\documentclass[runningheads]{llncs}
\usepackage{enumerate}
\usepackage{enumitem}
\usepackage[cmex10]{amsmath} 
\usepackage{amssymb} 
\usepackage{amsfonts}

\usepackage{amsthm}
  
\usepackage[mathcal]{euscript}
\usepackage[lined,boxed,commentsnumbered]{algorithm2e}

\usepackage{graphicx}
\usepackage{graphics} 
\usepackage{epsfig} 

\usepackage{color}
\makeatletter
\let\NAT@parse\undefined
\makeatother

\usepackage{cite}

\usepackage{nohyperref}
\usepackage{url}
\usepackage{comment}

\begin{document}
\title{Technical Report: Reactive Navigation in Partially Known Non-Convex Environments\thanks{This work was supported by AFRL grant FA865015D1845 (subcontract 6697371). The authors thank Dr. Omur Arslan for many formative discussions and for sharing his simulation and presentation infrastructure.}}
\titlerunning{Technical Report}
%
\author{Vasileios Vasilopoulos\inst{1} \and
Daniel E. Koditschek\inst{2}}
%
%
\institute{Department of Mechanical Engineering and Applied Mechanics, University of Pennsylvania, Philadelphia, PA 19104 \and
Department of Electrical and Systems Engineering, University of Pennsylvania, Philadelphia, PA 19104\\
\email{\{vvasilo,kod\}@seas.upenn.edu}}
\maketitle              
\begin{abstract}
This paper presents a provably correct method for robot navigation in 2D environments cluttered with familiar but unexpected non-convex, star-shaped obstacles as well as completely unknown, convex obstacles. We presuppose a limited range onboard sensor, capable of recognizing, localizing and (leveraging ideas from constructive solid geometry) generating online from its catalogue of the familiar, non-convex shapes an implicit representation of each one. These representations underlie an online change of coordinates to a completely convex model planning space wherein a previously developed online construction yields a provably correct reactive controller that is pulled back to the physically sensed representation to generate the actual robot commands. We extend the construction to differential drive robots, and suggest the empirical utility of the proposed control architecture using both formal proofs and numerical simulations.
\end{abstract}

\allowdisplaybreaks

\section{Introduction}

\subsection{Motivation and Prior Work}
Recent advances in the theory of sensor-based reactive navigation \cite{arslan_kod_WAFR2016} and its application to wheeled \cite{Arslan_Koditschek_2018} and legged \cite{vasilopoulos2017} robots promote its central role in provably correct architectures for increasingly complicated mobile manipulation tasks \cite{vasilopoulos2018,Vasilopoulos_Topping_Vega-Brown_Roy_Koditschek_2018}. The advance of the new theory \cite{arslan_kod_WAFR2016} over prior sensor-based collision avoidance schemes \cite{paranjape_etal_IJRR2015,johnson_hale_haynes_kod_SSRR2011,simmons_ICRA1996,fiorini_shiller_IJRR1998,vandenberg_guy_lin_manocha_ISRR2011,vandenberg_lin_manocha_ICRA2008,brock_khatib_ICRA1999,borenstein_koren_TRA1991,borenstein_koren_TSMC1989} was the additional guaranteed convergence to a designated goal which had theretofore only been established for reactive planners possessing substantial prior knowledge about the environment \cite{Majumdar_Tedrake_2017,rimon1992}. A key feature of these new (and other recent parallel \cite{Paternain_Koditschek_Ribeiro_2017,Ilhan_Johnson_Koditschek_2018}) approaches is that they trade away prior knowledge for the presumption of simplicity: unknown obstacles can be successfully negotiated in real time without losing global convergence guarantees if they are ``round'' (i.e., very strongly convex in a sense made precise in \cite{Arslan_Koditschek_2018}). The likely necessity of such simple geometry  for guaranteed safe convergence by a completely uninformed ``greedy'' reactive navigation planner is suggested by the result that a collision avoiding, distance-diminishing reactive navigation policy can reach arbitrarily placed goals in an unknown freespace {\em only if} all obstacles are ``round'' \cite[Proposition 14]{Arslan_Koditschek_2018}.

This paper offers a step toward elucidating the manner in which partial knowledge may suffice to inform safe, convergent, reactive navigation in geometrically more interesting environments. Growing experience negotiating learned \cite{henry_vollmer_ferris_fox_ICRA2010} or estimated \cite{karaman_frazzoli_ICRA2012,trautman_ma_murray_krause_IJRR2015} environments suggests that reasonable statistical priors may go a long way toward provable stochastic navigation. But in this work we are interested in what sort of deterministic guarantees may be possible. Recent developments in semantic SLAM \cite{Bowman2017} and object pose and triangular mesh extraction using convolutional neural net architectures \cite{Kar_Tulsiani_Carreira_Malik_2015,Kong_Lin_Lucey_2017,Pavlakos2017} now provide an avenue for incorporating partial prior knowledge within a deterministic framework well suited to the vector field planning methods reviewed above.

\subsection{Contributions and Organization of the Paper}
We consider the navigation problem in a 2D workspace cluttered with unknown convex obstacles, along with ``familiar'' non-convex, star-shaped obstacles \cite{Rimon_Koditschek_1989} that belong to classes of known geometries, but whose number and placement are unknown, awaiting discovery at execution time. We assume a limited range onboard sensor, a sufficient  margin separating all obstacles from each other and the goal, and a catalogue of known star-shaped sets, along with a ``mapping oracle'' for their online identification and localization in the physical workspace. These ingredients suggest a representation of the environment taking the form of a ``multi-layer'' triple of topological spaces whose realtime interaction can be exploited to integrate the geometrically naive sensor driven methods of \cite{arslan_kod_WAFR2016} with the offline memorized geometry sensitive methods of \cite{rimon1992}. Specifically, we adapt the construction of \cite{Rimon_Koditschek_1989} to generate a realtime smooth change of coordinates (a {\em diffeomorphism}) of the mapped layer of the environment into a (locally) topologically equivalent but geometrically more favorable model layer relative to which the reactive methods of \cite{arslan_kod_WAFR2016} can be directly applied. We prove that the conjugate vector field defined by pulling back the reactive model space planner through this diffeomorphism induces a vector field on the robot's physical configuration space that inherits the same formal guarantees of obstacle avoidance and convergence. We extend the construction to the case of a differential drive robot, by pulling back the extended field over planar rigid transformations introduced for this purpose in \cite{arslan_kod_WAFR2016} through a suitable polar coordinate transformation of the tangent lift of our original planar diffeomorphism and demonstrate, once again, that the physical differential drive robot inherits the same obstacle avoidance and convergence properties as those guaranteed for the geometrically simple model robot \cite{arslan_kod_WAFR2016}. Finally, to better support online implementation of these constructions, we adopt modular methods for implicit description of geometric shape \cite{shapiro2007}.

The paper is organized as follows. Section \ref{sec:problem_formulation} describes the problem and establishes our assumptions. Section \ref{sec:geometric_transformation} describes the physical, mapped and model planning layers used in the constructed diffeomorphism between the mapped and model layers, whose properties are established next. Based on these results, Section \ref{sec:reactive_controller} describes our control approach both for fully actuated and differential drive robots. Section \ref{sec:simulations} presents a variety of illustrative numerical studies and Section \ref{sec:conclusion} concludes by summarizing our findings and presenting ideas for future work. Finally, Appendix \ref{appendix:proofs} includes the proofs of our main results, Appendix \ref{appendix:rfunctions} sketches the ideas from computational geometry \cite{shapiro2007} underlying our modular construction of implicit representations of polygonal obstacles, and Appendix \ref{appendix:calculation_jacobian} includes some technical details on the calculation of the diffeomorphism jacobian for differential drive robots.

\section{Problem Formulation}
\label{sec:problem_formulation}

We consider a disk-shaped robot with radius $r>0$, centered at $\mathbf{x} \in \mathbb{R}^2$, navigating a closed, compact workspace $\mathcal{W} \subset \mathbb{R}^2$, with known convex boundary $\partial \mathcal{W}$. The robot is assumed to possess a sensor with fixed range $R$, capable of recognizing ``familiar'' objects, as well as estimating the distance of the robot to nearby obstacles\footnote{We refer the reader to an example of existing technology \cite{pointcloud_to_lidar} generating 2D LIDAR scans from 3D point clouds for such an approach.}. 

The workspace is cluttered by an unknown number of fixed, disjoint obstacles, denoted by $\mathcal{O}:=(O_1,O_2,\ldots)$. We adopt the notation in \cite{arslan_kod_WAFR2016} and define the {\em freespace} as
\begin{equation}
\mathcal{F} := \left\{ \mathbf{x} \in \mathcal{W} \, \Big| \, \overline{B(\mathbf{x},r)} \subseteq \mathcal{W} \, \backslash \, \bigcup_i O_i \right\} \label{eq:free_space}
\end{equation}
where $B(\mathbf{x},r)$ is the open ball centered at $\mathbf{x}$ with radius $r$, and $\overline{B(\mathbf{x},r)}$ denotes its closure. To simplify our notation, we neglect the robot dimensions, by dilating each obstacle in $\mathcal{O}$ by $r$, and assume that the robot operates in $\mathcal{F}$. We denote the set of dilated obstacles by $\tilde{\mathcal{O}}$.

Although none of the positions of any obstacles in $\tilde{\mathcal{O}}$ are \`{a}-priori known, a subset $\tilde{\mathcal{O}}^* \subseteq \tilde{\mathcal{O}}$ of these obstacles is assumed to be ``familiar'' in the sense of having an \`{a}-priori known, readily recognizable star-shaped geometry \cite{Rimon_Koditschek_1989} (i.e., belonging to a known catalogue of star-shaped {\it geometry classes}), which the robot can efficiently identify and localize instantaneously from online sensory measurement. Although the implementation of such a sensory apparatus lies well beyond the scope of the present paper, recent work on semantic SLAM \cite{Bowman2017} provides an excellent example with empirically demonstrated technology for achieving this need for localizing, identifying and keeping track of all the familiar obstacles encountered in the otherwise unknown environment. The \`{a}-priori unknown center of each catalogued star-shaped obstacle $\tilde{O}^*_i$ is denoted $\mathbf{x}^*_i$. Similarly to \cite{rimon1992}, each star-shaped obstacle $\tilde{O}^*_i \in \tilde{\mathcal{O}}^*$ can be described by an {\it obstacle function}, a real-valued map providing an implicit representation of the form
\begin{equation}
\tilde{O}^*_i = \{ \mathbf{x} \in \mathbb{R}^2 \, | \, \beta_i(\mathbf{x}) \leq 0 \}
\end{equation}
which the robot must construct online from the catalogued geometry, after it has localized $\tilde{O}^*_i$. The remaining obstacles $\tilde{\mathcal{O}}_{convex}:=\tilde{\mathcal{O}}\backslash\tilde{\mathcal{O}}^*$ are are assumed to be strictly convex but are in all other regards (location and specific shape) completely unknown to the robot, while nevertheless satisfying a curvature condition given in \cite[Assumption 2]{arslan_kod_WAFR2016}. 

For the obstacle functions, we require the technical assumptions introduced in \cite[Appendix III]{rimon1992}, outlined as follows.
\begin{assumption} \label{assumption:epsilon}
The obstacle functions satisfy the following requirements
\begin{enumerate}[label=\alph*)]
\item For each $\tilde{O}^*_i \in \tilde{\mathcal{O}}^*$, there exists $\varepsilon_{1i}>0$ such that for any two obstacles $\tilde{O}^*_i,\tilde{O}^*_j \in \tilde{\mathcal{O}}^*$
\begin{equation}
\{\mathbf{x} \, | \, \beta_i(\mathbf{x}) \leq \varepsilon_{1i} \} \bigcap \{\mathbf{x} \, | \, \beta_j(\mathbf{x}) \leq \varepsilon_{1j} \} = \emptyset
\end{equation}
i.e., the ``thickened boundaries'' of any two stars still do not overlap.
\item For each $\tilde{O}^*_i \in \tilde{\mathcal{O}}^*$, there exists $\varepsilon_{2i}>0$ such that the set $\{\mathbf{x} \, | \, \beta_i(\mathbf{x}) \leq \varepsilon_{2i} \}$ does not contain the goal $\mathbf{x}_d \in \mathcal{F}$ and does not intersect with any other obstacle in $\tilde{\mathcal{O}}_{convex}$.
\item For each obstacle function $\beta_i$, there exists a pair of positive constants $(\delta_i,\varepsilon_{3i})$ satisfying the inner product condition\footnote{A brief discussion on this condition is given in Appendix \ref{appendix:rfunctions}.}
\begin{equation}
(\mathbf{x}-\mathbf{x}^*_i)^\top \nabla \beta_i(\mathbf{x}) \geq \delta_i
\end{equation}
for all $\mathbf{x} \in \mathbb{R}^2$ such that $\beta_i(\mathbf{x}) \leq \varepsilon_{3i}$.
\end{enumerate}
\end{assumption}
For each obstacle $\tilde{O}^*_i \in \tilde{\mathcal{O}}^*$, we then define $\varepsilon_i = \min\{\varepsilon_{1i},\varepsilon_{2i},\varepsilon_{3i}\}$. Finally, we will assume that the range of the sensor $R$ satisfies $R >> \varepsilon_i$ for all $i$.

Based on these assumptions and further positing first-order, fully-actuated robot dynamics $\dot{\mathbf{x}} = \mathbf{u}(\mathbf{x})$, the problem consists of finding a Lipschitz continuous controller $\mathbf{u}:\mathcal{F} \rightarrow \mathbb{R}^2$, that leaves the freespace $\mathcal{F}$ positively invariant and asymptotically steers almost all configurations in $\mathcal{F}$ to the given goal $\mathbf{x}_d \in \mathcal{F}$.

\section{Multi-layer Representation of the Environment and Its Associated Transformations}
\label{sec:geometric_transformation}
In this Section, we introduce associated notation for, and transformations between three distinct representations of the environment that we will refer to as planning ``layers'' and use in the construction of our algorithm. Fig. \ref{fig:diffeo_idea} illustrates the role of these layers and the transformations that relate them in constructing and analyzing a realtime generated vector field that guarantees safe passage to the goal. The new technical contribution is an adaptation  of the methods of \cite{rimon1992} to the construction of a diffeomorphism, $\mathbf{h}$, where the requirement for fast, online performance demands an algorithm that is as simple as possible and with few tunable parameters. Hence, since the reactive controller in \cite{arslan_kod_WAFR2016} is designed to (provably) handle convex shapes, sensed  obstacles not recognized by the semantic SLAM process are simply assumed to be convex (implemented by designing $\mathbf{h}$ to resolve to the identity transformation in the  neighborhood of ``unfamiliar'' objects) and the control response defaults to that prior construction.  

\begin{figure}[t]
\centering
\includegraphics[width=\textwidth]{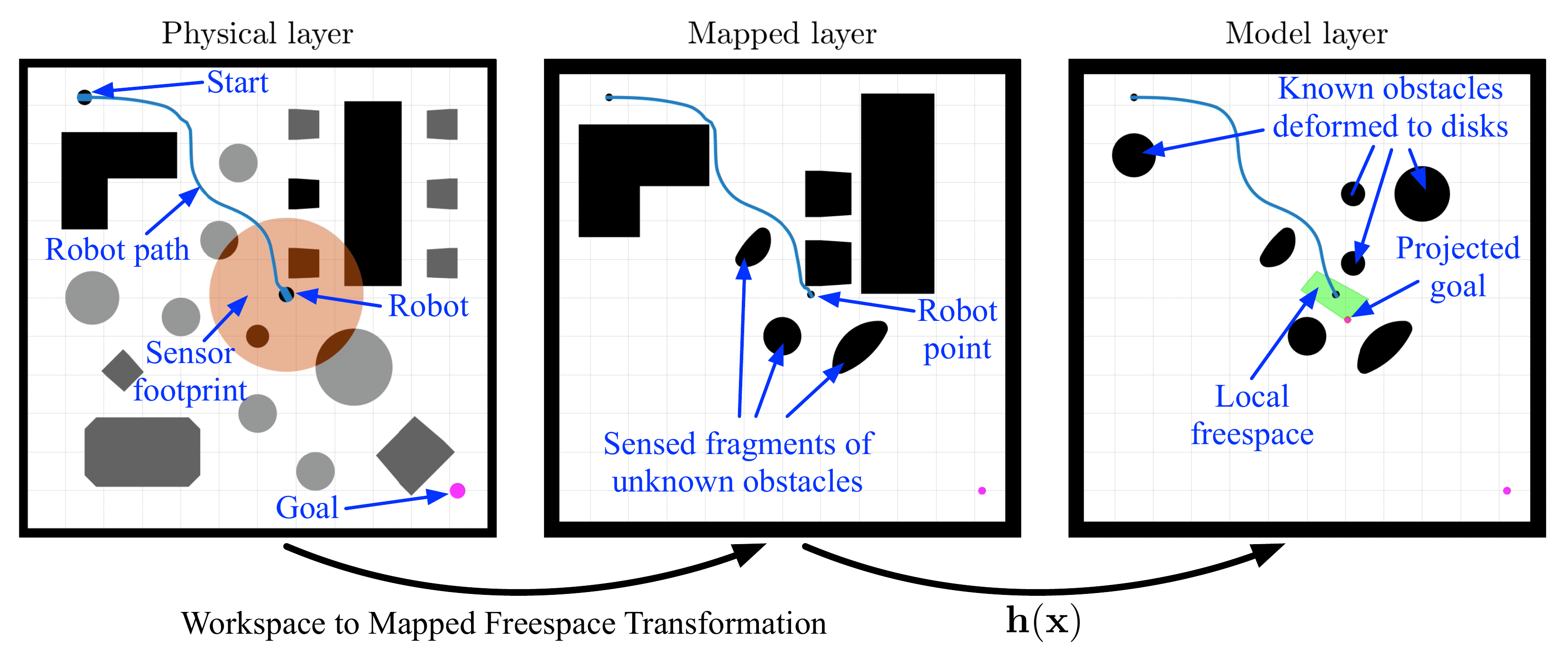}
\caption{Snapshot Illustration of Key Ideas. The robot in the physical layer (left frame, depicting in blue the robot's placement in the workspace along with the prior trajectory of its centroid) containing both familiar objects of known geometry but unknown location (dark grey) and unknown obstacles (light grey), moves towards a goal and discovers obstacles (black) with an onboard sensor of limited range (orange disk). These obstacles are localized and stored permanently in the mapped layer (middle frame, depicting in blue the robot's placement as a point in freespace rather than its body in the workspace) if they have familiar geometry or temporarily, with just the corresponding sensed fragments, if they are unknown. An online map $\mathbf{h}(\mathbf{x})$ is then constructed (Section \ref{sec:geometric_transformation}), from the mapped layer to a geometrically simple model layer (right frame, now depicting the robot's placement and prior tractory amongst the $\mathbf{h}$-deformed convex images of the mapped obstacles). A doubly reactive control scheme for convex environments \cite{arslan_kod_WAFR2016} defines a vector field on the model layer which is pulled back in realtime through the diffeomorphism to generate the input in the physical layer (Section \ref{sec:reactive_controller}).} \label{fig:diffeo_idea}
\end{figure}

\subsection{Description of Planning Layers}
\subsubsection{Physical Layer}
The {\it physical layer} is a complete description of the geometry of the unknown actual world and while inaccessible to the robot is used for purposes of analysis. It describes the actual workspace $\mathcal{W}$, punctured with the obstacles $\mathcal{O}$. This gives rise to the freespace $\mathcal{F}$, given in \eqref{eq:free_space}, consisting of all placements of the robot's centroid that entail no intersections of its body with any obstacles. The robot navigates this layer, and discovers and localizes new obstacles, which are then stored in its {\it semantic map} if their geometry is familiar.

\subsubsection{Mapped Layer}
The {\it mapped layer} $\mathcal{F}_{map}$ has the same boundary as $\mathcal{F}$ (i.e. $\partial \mathcal{F}_{map}:=\partial \mathcal{F}$) and records the robot's evolving information about the environment aggregated from the raw sensor data about the observable portions of $N \geq 0$ unrecognized (and therefore, presumed convex) obstacles $\{\tilde{O}_1,\ldots,\tilde{O}_N\} \subseteq \tilde{\mathcal{O}}_{convex}$, together with the inferred star centers $\mathbf{x}^*_j$ and obstacle functions $\beta_j$ of $M \geq 0$ star-shaped obstacles $\{\tilde{O}^*_1,\ldots,\tilde{O}^*_M\} \subseteq \tilde{\mathcal{O}}^*$, that are instantiated at the moment the sensory data triggers the ``memory'' that identifies and localizes a familiar obstacle. It is important to note that the star environment is constantly updated, both by discovering and storing new star-shaped obstacles in the semantic map and by discarding old information and storing new information regarding obstacles in $\tilde{\mathcal{O}}_{convex}$. In this representation, the robot is treated as a point particle, since all obstacles are dilated by $r$ in the passage from the workspace to the freespace representation of valid placements.

\subsubsection{Model Layer}
The {\it model layer} $\mathcal{F}_{model}$ has the same boundary as $\mathcal{F}$ (i.e. $\partial \mathcal{F}_{model}:=\partial \mathcal{F}$) and consists of a collection of $M$ Euclidean disks, each centered at one of the mapped star centers, $\mathbf{x}^*_j, j=1,\ldots,M$, and copies of the sensed fragments of the $N$ unrecognized visible convex obstacles in $\tilde{\mathcal{O}}_{convex}$. The radii $\{\rho_1,\ldots,\rho_M\}$ of the $M$ disks are chosen so that $\overline{B(\mathbf{x}^*_j,\rho_j)} \subseteq \{ \mathbf{x} \, | \, \beta_j(\mathbf{x}) < 0 \}$, as in \cite{rimon1992}.

This metric convex sphere world comprises the data generating the doubly reactive algorithm of \cite{arslan_kod_WAFR2016}, which will be applied to the physical robot via the online generated change of coordinates between the mapped layer and the model layer to be now constructed.

\subsection{Description of the $C^\infty$ Switches}
In order to simplify the diffeomorphism construction, we depart from the construction of analytic switches \cite{Rimon_Koditschek_1989} and rely instead on the $C^\infty$ function $\zeta:\mathbb{R} \rightarrow \mathbb{R}$ \cite{hirsch_1976} described by
\begin{equation}
\zeta(\chi) = \left\{ \begin{matrix}
e^{-1/\chi}, & \quad \chi>0 \\
0,  & \quad \chi \leq 0
\end{matrix}\right.
\end{equation}
with derivative
\begin{equation} \label{eq:zeta_derivative}
\zeta'(\chi) = \left\{ \begin{matrix}
\frac{\zeta(\chi)}{\chi^{2}}, & \quad \chi>0 \\
0,  & \quad \chi \leq 0
\end{matrix}\right.
\end{equation}
Based on that function, we can then define the $C^\infty$ switches for each star-shaped obstacle $\tilde{O}^*_j$ in the semantic map as
\begin{equation}
\sigma_j(\mathbf{x}) = \eta_j \circ \beta_j(\mathbf{x}), \quad j=1,\ldots,M \label{eq:sigmaj}
\end{equation}
with $\eta_j(\chi) = \zeta(\varepsilon_j-\chi)/\zeta(\varepsilon_j)$ and $\varepsilon_j$ given according to Assumption \ref{assumption:epsilon}. The gradient of the switch $\sigma_j$ is given by
\begin{equation} \label{eq:grad_sigmaj}
\nabla \sigma_j(\mathbf{x}) = (\eta_j' \circ \beta_j(\mathbf{x})) \cdot \nabla \beta_j(\mathbf{x})
\end{equation}
Finally, we define
\begin{equation}
\sigma_d(\mathbf{x}) = 1-\sum_{j=1}^M \sigma_j(\mathbf{x})
\end{equation}

Using the above construction, it is easy to see that $\sigma_j(\mathbf{x}) = 1$ on the boundary of the $j$-th obstacle and $\sigma_j(\mathbf{x}) = 0$ when $\beta_j(\mathbf{x}) > \varepsilon_j$ for each $j=1,\ldots,M$. Based on Assumption \ref{assumption:epsilon} and the choice of $\varepsilon_j$ for each $j$, we are, therefore, led to the following results.
\begin{lemma} \label{lemma:nonzero_switches}
At any point $\mathbf{x} \in \mathcal{F}_{map}$, at most one of the switches $\{ \sigma_1, \ldots, \sigma_M \}$ can be nonzero.
\end{lemma}
\begin{corollary}
The set $\left\{\sigma_1, \ldots, \sigma_M, \sigma_d \right\}$ defines a partition of unity over $\mathcal{F}_{map}$.
\end{corollary}

\subsection{Description of the Star Deforming Factors}
The deforming factors are the functions $\nu_j(\mathbf{x}):\mathcal{F}_{map} \rightarrow \mathbb{R}, j=1,\ldots,M$, responsible for transforming each star-shaped obstacle into a disk in $\mathbb{R}^2$. Once again, we use here a slightly different construction than \cite{Rimon_Koditschek_1989}, in that the value of each deforming factor $\nu_j$ at a point $\mathbf{x}$ does not depend on the value of $\beta_j(\mathbf{x})$. Namely, the deforming factors are given based on the desired final radii $\rho_j, j=1,\ldots,M$ as
\begin{equation}
\nu_j(\mathbf{x}) = \frac{\rho_j}{||\mathbf{x}-\mathbf{x}^*_j||} \label{eq:deforming_factors}
\end{equation}
We also get
\begin{equation}
\nabla \nu_j(\mathbf{x}) = - \frac{\rho_j}{||\mathbf{x}-\mathbf{x}^*_j||^3} (\mathbf{x}-\mathbf{x}^*_j) \label{eq:grad_nu}
\end{equation}

\subsection{The Map Between the Mapped and the Model Layer}
\subsubsection{Construction}
The map for $M$ star-shaped obstacles centered at $\mathbf{x}^*_j, j = 1, \ldots, M$ is described by a function $\mathbf{h}: \mathcal{F}_{map} \rightarrow \mathcal{F}_{model}$ given by
\begin{equation}
\mathbf{h}(\mathbf{x}) = \sum_{j=1}^M \sigma_j(\mathbf{x})\left[\nu_j(\mathbf{x})(\mathbf{x}-\mathbf{x}^*_j) + \mathbf{x}^*_j \right] + \sigma_d(\mathbf{x}) \mathbf{x} \label{eq:map}
\end{equation}

Note that the $N$ visible convex obstacles $\{\tilde{O}_1,\ldots,\tilde{O}_N\} \subseteq \tilde{\mathcal{O}}_{convex}$ are not considered in the construction of the map. Since the reactive controller used in the model space $\mathcal{F}_{model}$ can handle convex obstacles and there is enough separation between convex and star-shaped obstacles according to Assumption \ref{assumption:epsilon}-(b), we can ``transfer'' the geometry of those obstacles directly in the model space using the identity transformation.

Finally, note that Assumption \ref{assumption:epsilon}-(b) implies that $\mathbf{h}(\mathbf{x}_d) = \mathbf{x}_d$, since the target location is assumed to be sufficiently far from all star-shaped obstacles.

Based on the construction of the map $\mathbf{h}$, the jacobian $D_\mathbf{x}\mathbf{h}$ at any point $\mathbf{x} \in \mathcal{F}_{map}$ is given by
\begin{equation}
D_\mathbf{x}\mathbf{h} = \sum_{j=1}^M \left\{ \sigma_j(\mathbf{x}) \nu_j(\mathbf{x}) \mathbf{I} + (\mathbf{x}-\mathbf{x}^*_j) \left[\sigma_j(\mathbf{x}) \nabla \nu_j(\mathbf{x})^\top + (\nu_j(\mathbf{x})-1) \nabla \sigma_j(\mathbf{x})^\top \right] \right\} + \sigma_d(\mathbf{x}) \mathbf{I} \label{eq:map_differential}
\end{equation}

\subsubsection{Qualitative Properties of the Map} 
We first verify that the construction is a smooth change of coordinates between the mapped and the model layers.
\begin{lemma} \label{lemma:smooth_map}
The map $\mathbf{h}$ from $\mathcal{F}_{map}$ to $\mathcal{F}_{model}$ is smooth.
\end{lemma}
\begin{proof}
Included in Appendix \ref{appendix_sec:section_geometric}.
\end{proof}

\begin{proposition} \label{proposition:diffeomorphism}
The map $\mathbf{h}$ is a $C^\infty$ diffeomorphism between $\mathcal{F}_{map}$ and $\mathcal{F}_{model}$.
\end{proposition}
\begin{proof}
Included in Appendix \ref{appendix_sec:section_geometric}.
\end{proof}

\subsubsection{Implicit representation of obstacles}
To implement the diffeomorphism between $\mathcal{F}_{map}$ and $\mathcal{F}_{model}$, shown in \eqref{eq:map}, we rely on the existence of a smooth obstacle function $\beta_j(\mathbf{x})$ for each star-shaped obstacle $j=1,\ldots,M$ stored in the semantic map. Since recently developed technology \cite{Pavlakos2017,Kar_Tulsiani_Carreira_Malik_2015,Kong_Lin_Lucey_2017} provides means of performing obstacle identification in the form of triangular meshes, in this work we focus on polygonal obstacles on the plane and derive implicit representations using so called ``R-functions'' from the constructive solid geometry literature \cite{shapiro2007}. In Appendix \ref{appendix:rfunctions}, we describe the method used for the construction of such implicit functions for polygonal obstacles that have the desired property of being analytic everywhere except for the polygon vertices. For the construction, we assume that the sensor has already identified, localized and included each discovered star-shaped obstacle in $\mathcal{F}_{map}$; i.e., it has determined its pose in $\mathcal{F}_{map}$, given as a rotation $\mathbf{R}_j$ of its vertices on the plane followed by a translation of its center $\mathbf{x}^*_j$, and that the corresponding polygon has already been dilated by $r$ for inclusion in $\mathcal{F}_{map}$.

\section{Reactive Controller}
\label{sec:reactive_controller}

\subsection{Reactive Controller for Fully Actuated Robots}
\subsubsection{Construction}
First, we consider a fully actuated particle with state $\mathbf{x} \in \mathcal{F}_{map}$, whose dynamics are described by
\begin{equation}
\dot{\mathbf{x}} = \mathbf{u}
\end{equation}
The dynamics of the fully actuated particle in $\mathcal{F}_{model}$ with state $\mathbf{y} \in \mathcal{F}_{model}$ are described by $\dot{\mathbf{y}} = \mathbf{v}(\mathbf{y})$ with the control $\mathbf{v}(\mathbf{y})$ given in \cite{arslan_kod_WAFR2016} as\footnote{Here $\mathrm{\Pi}_C(\mathbf{q})$ denotes the metric projection of $\mathbf{q}$ on a convex set $C$.}
\begin{equation}
\mathbf{v}(\mathbf{y}) = -k \, \left( \mathbf{y}-\mathrm{\Pi}_{\mathcal{LF}(\mathbf{y})}(\mathbf{x}_d) \right) 
\end{equation}
Here, the convex {\it local freespace} for $\mathbf{y}$, $\mathcal{LF}(\mathbf{y}) \subset \mathcal{F}_{model}$, is defined as in \cite[Eqn. (30)]{arslan_kod_WAFR2016}.
Using the diffeomorphism construction in \eqref{eq:map} and its jacobian in \eqref{eq:map_differential}, we construct our controller as the vector field $\mathbf{u}:\mathcal{F}_{map} \rightarrow T\mathcal{F}_{map}$ given by
\begin{equation}
\mathbf{u}(\mathbf{x}) = [D_\mathbf{x}\mathbf{h}]^{-1} \, \cdot \left(\mathbf{v} \circ \mathbf{h}(\mathbf{x})\right) \label{eq:controller}
\end{equation}

\subsubsection{Qualitative Properties}
First of all, if the range of the virtual LIDAR sensor used to construct $\mathcal{LF}(\mathbf{y})$ in the model layer is smaller than $R$, the vector field $\mathbf{u}$ is Lipschitz continuous since $\mathbf{v}(\mathbf{y})$ is shown to be Lipschitz continuous in \cite{arslan_kod_WAFR2016} and $\mathbf{y}=\mathbf{h}(\mathbf{x})$ is a smooth change of coordinates. We are led to the following result.
\begin{corollary}
The vector field $\mathbf{u}: \mathcal{F}_{map} \rightarrow T\mathcal{F}_{map}$ generates a unique continuously differentiable partial flow.
\end{corollary}

To ensure completeness (i.e. absence of finite time escape through boundaries in $\mathcal{F}_{map}$) we must verify that the robot never collides with any obstacle in the environment, i.e., leaves its freespace positively invariant.

\begin{proposition} \label{proposition:positive_invariance}
The freespace $\mathcal{F}_{map}$ is positively invariant under the law \eqref{eq:controller}.
\end{proposition}
\begin{proof}
Included in Appendix \ref{appendix_sec:section_controller}.
\end{proof}

\begin{lemma} \label{lemma:equilibrium_stability}
\begin{enumerate}
\item The set of stationary points of control law \eqref{eq:controller} is given as
$\{ \mathbf{x}_d\} \bigcup \{\mathbf{h}^{-1}(\mathbf{s}_j)\}_{j \in \{1,\ldots,M\}} \bigcup_{i=1}^N \mathcal{G}_i$, where\footnote{Here $d(A,B)=\inf\{||\mathbf{a}-\mathbf{b}|| \, | \, \mathbf{a} \in A, \mathbf{b} \in B\}$ denotes the distance between two sets $A,B$.}
\begin{subequations} \label{eq:saddles}
\begin{eqnarray}
& \mathbf{s}_j = \mathbf{x}^*_j - \rho_j \dfrac{\mathbf{x}_d-\mathbf{x}^*_j}{|| \mathbf{x}_d-\mathbf{x}^*_j ||} \label{eq:saddles_disks} \\
& \mathcal{G}_i := \left\{ \mathbf{q} \in \mathcal{F}_{map} \, \Big | \, d(\mathbf{q},O_i)=r, \dfrac{(\mathbf{q}-\mathrm{\Pi}_{\overline{O}_i}(\mathbf{q}))^\top(\mathbf{q}-\mathbf{x}_d)}{||\mathbf{q}-\mathrm{\Pi}_{\overline{O}_i}(\mathbf{q})|| \, ||\mathbf{q}-\mathbf{x}_d||} = 1 \right\} \label{eq:saddles_convex}
\end{eqnarray}
\end{subequations}
with $j$ spanning the $M$ star-shaped obstacles in $\mathcal{F}_{map}$ and $i$ spanning the $N$ convex obstacles in $\mathcal{F}_{map}$.
\item The goal $\mathbf{x}_d$ is the only locally stable equilibrium of control law \eqref{eq:controller} and all the other stationary points $\{\mathbf{h}^{-1}(\mathbf{s}_j)\}_{j \in \{1,\ldots,M\}} \bigcup_{i=1}^N \mathcal{G}_i$, each associated with an obstacle, are nondegenerate saddles.
\end{enumerate}
\end{lemma}
\begin{proof}
Included in Appendix \ref{appendix_sec:section_controller}.
\end{proof}

\begin{proposition} \label{proposition:attraction}
The goal location $\mathbf{x}_d$ is an asymptotically stable equilibrium of \eqref{eq:controller}, whose region of attraction includes the freespace $\mathcal{F}_{map}$ excepting a set of measure zero.
\end{proposition}
\begin{proof}
Included in Appendix \ref{appendix_sec:section_controller}.
\end{proof}

We can now immediately conclude the following central summary statement.
\begin{theorem}
The reactive controller in \eqref{eq:controller} leaves the freespace $\mathcal{F}_{map}$ positively invariant, and its unique continuously differentiable flow, starting at almost any robot placement $\mathbf{x} \in \mathcal{F}_{map}$, asymptotically reaches the goal location $\mathbf{x}_d$, while strictly decreasing $|| \mathbf{h}(\mathbf{x})-\mathbf{x}_d||$ along the way.
\end{theorem}

\subsection{Reactive Controller for Differential Drive Robots}
In this Section, we extend our reactive controller to the case of a differential drive robot, whose state is $\overline{\mathbf{x}}:=(\mathbf{x},\psi) \in \mathcal{F}_{map} \times S^1 \subset SE(2)$, and its dynamics are given by\footnote{We use the ordered set notation $(*,*,\ldots)$ and the matrix notation $\begin{bmatrix}* & * & \ldots \end{bmatrix}^\top$ for vectors interchangeably.}
\begin{equation}
\dot{\overline{\mathbf{x}}} = \mathbf{B}(\psi) \overline{\mathbf{u}} \label{eq:unicycle_dynamics}
\end{equation}
with $\mathbf{B}(\psi) = \begin{bmatrix}
\cos\psi & \sin\psi & 0 \\ 0 & 0 & 1
\end{bmatrix}^\top$ and $\overline{\mathbf{u}} = (v,\omega)$ with $v \in \mathbb{R}$ and $\omega \in \mathbb{R}$ the linear and angular input respectively. We will follow a similar procedure to the fully actuated case; we begin by describing a smooth diffeomorphism $\overline{\mathbf{h}}:\mathcal{F}_{map} \times S^1 \rightarrow \mathcal{F}_{model} \times S^1$ and then we establish the results about the controller.

\subsubsection{Construction and Properties of the $SE(2)$ Diffeomorphism}
We construct our map $\overline{\mathbf{h}}$ from $\mathcal{F}_{map} \times S^1$ to $\mathcal{F}_{model} \times S^1$ as
\begin{equation}
\overline{\mathbf{y}}=(\mathbf{y},\varphi) = \overline{\mathbf{h}}(\overline{\mathbf{x}}):=\left( \mathbf{h}(\mathbf{x}), \xi(\overline{\mathbf{x}})\right) \label{eq:map_se2}
\end{equation}
with $\overline{\mathbf{x}}=(\mathbf{x},\psi) \in \mathcal{F}_{map} \times S^1$, $\overline{\mathbf{y}}:=(\mathbf{y},\varphi) \in \mathcal{F}_{model} \times S^1$ and
\begin{equation}
\varphi = \xi(\overline{\mathbf{x}}) := \angle \left(\mathbf{e}(\overline{\mathbf{x}}) \right)\label{eq:phi_definition}
\end{equation}
Here, $\angle\mathbf{e}:=\text{atan2}(e_2,e_1)$ and
\begin{equation}
\mathbf{e}(\overline{\mathbf{x}}) = \mathrm{\Pi}_{\mathbf{y}} \cdot D_{\overline{\mathbf{x}}}\overline{\mathbf{h}} \cdot \mathbf{B}(\psi) \cdot \begin{bmatrix}
1 \\ 0
\end{bmatrix} = D_\mathbf{x}\mathbf{h} \begin{bmatrix}
\cos\psi \\ \sin\psi
\end{bmatrix} \label{eq:e}
\end{equation}
with $\mathrm{\Pi}_\mathbf{y}$ denoting the projection onto the first two components. The reason for choosing $\varphi$ as in \eqref{eq:phi_definition} will become evident in the next paragraph, in our effort to control the equivalent differential drive robot dynamics in $\mathcal{F}_{model}$.

\begin{proposition} \label{proposition:diffeo_se2}
The map $\overline{\mathbf{h}}$ in \eqref{eq:map_se2} is a $C^\infty$ diffeomorphism from $\mathcal{F}_{map} \times S^1$ to $\mathcal{F}_{model} \times S^1$.
\end{proposition}
\begin{proof}
Included in Appendix \ref{appendix_sec:section_controller}.
\end{proof}

\subsubsection{Construction of the Reactive Controller}
Using \eqref{eq:map_se2}, we can find the pushforward of the differential drive robot dynamics in \eqref{eq:unicycle_dynamics} as
\begin{equation}
\dot{\overline{\mathbf{y}}}=\begin{bmatrix}
\dot{\mathbf{y}} \\ \dot{\varphi}
\end{bmatrix} = \frac{d}{dt} \begin{bmatrix}
\mathbf{h}(\mathbf{x}) \\ \xi(\overline{\mathbf{x}})
\end{bmatrix} = \left[D_{\overline{\mathbf{x}}} \overline{\mathbf{h}} \circ \overline{\mathbf{h}}^{-1}(\overline{\mathbf{y}}) \right] \cdot \left( \mathbf{B} \circ \overline{\mathbf{h}}^{-1}(\overline{\mathbf{y}}) \right) \cdot \overline{\mathbf{u}} \label{eq:vector_field_se2}
\end{equation}
Based on the above, we can then write 
\begin{equation}
\dot{\overline{\mathbf{y}}}=\begin{bmatrix}
\dot{\mathbf{y}} \\ \dot{\varphi}
\end{bmatrix} = \frac{d}{dt} \begin{bmatrix}
\mathbf{h}(\mathbf{x}) \\ \xi(\overline{\mathbf{x}})
\end{bmatrix} = \mathbf{B}(\varphi) \overline{\mathbf{v}} \label{eq:unicycle_dynamics_se2}
\end{equation}
with $\overline{\mathbf{v}} = (\hat{v},\hat{\omega})$, and the inputs $(\hat{v},\hat{\omega})$ related to $(v,\omega)$ through
\begin{align}
& \hat{v} = ||\mathbf{e}(\overline{\mathbf{x}})|| \, v \label{eq:reference_input_1}\\
& \hat{\omega} = v D_\mathbf{x}\xi \begin{bmatrix}
\cos\psi \\ \sin\psi
\end{bmatrix} + \dfrac{\partial \xi}{\partial \psi} \omega \label{eq:reference_input_2}
\end{align}
with $D_\mathbf{x}\xi = \begin{bmatrix}
\frac{\partial \xi}{\partial x} & \frac{\partial \xi}{\partial y}
\end{bmatrix}$. The calculation of $D_\mathbf{x}\xi$ can be tedious, since it involves derivatives of elements of $D_\mathbf{x}\mathbf{h}$, and is included in Appendix \ref{appendix:calculation_jacobian}.

Hence, we have found equivalent differential drive robot dynamics, defined on $\mathcal{F}_{model} \times S^1$. The idea now is to use the control strategy in \cite{arslan_kod_WAFR2016} for the dynamical system in \eqref{eq:unicycle_dynamics_se2} to find {\em reference inputs} $\hat{v},\hat{\omega}$, and then use \eqref{eq:reference_input_1}, \eqref{eq:reference_input_2} to find the {\em actual inputs} $v,\omega$ that achieve those reference inputs as
\begin{subequations} \label{eq:control_unicycle}
\begin{eqnarray}
& v =\dfrac{\hat{v}}{||\mathbf{e}(\overline{\mathbf{x}})||} \\
& \omega = \left(\dfrac{\partial \xi}{\partial \psi}\right)^{-1} \left(\hat{\omega}-v D_\mathbf{x}\xi \begin{bmatrix}
\cos\psi \\ \sin\psi
\end{bmatrix} \right)
\end{eqnarray}
\end{subequations}
Namely, our reference inputs $\hat{v}$ and $\hat{\omega}$ inspired by \cite{arslan_kod_WAFR2016,astolfi_1999} are given as\footnote{In \eqref{eq:map_se2}, we construct a diffeomorphism $\overline{\mathbf{h}}$ between $\mathcal{F}_{map} \times S^1$ and $\mathcal{F}_{model} \times S^1$. However, for practical purposes, we deal only with one specific chart of $S^1$ in our control structure, described by the angles $(-\pi,\pi]$. As shown in \cite{astolfi_1999}, the discontinuity at $\pm \pi$ does not induce a discontinuity in our controller due to the use of the $\text{atan}$ function in \eqref{eq:reference_control_omega}. On the contrary, with the use of \eqref{eq:reference_control_omega} as in \cite{astolfi_1999,arslan_kod_WAFR2016}, the robot never changes heading in $\mathcal{F}_{model}$, which implies that the generated trajectories both in $\mathcal{F}_{model}$ and (by the properties of the diffeomorphism $\overline{\mathbf{h}}$) in $\mathcal{F}_{map}$ have no cusps, even though the robot might change heading in $\mathcal{F}_{map}$ because of the more complicated nature of the function $\xi$ in \eqref{eq:phi_definition}.}
\begin{subequations} \label{eq:control_unicycle_reference}
\begin{eqnarray}
& \hat{v} = -k \begin{bmatrix}
\cos\varphi \\ \sin\varphi
\end{bmatrix}^\top\left(\mathbf{y}-\mathrm{\Pi}_{\mathcal{LF}(\mathbf{y}) \cap H_\parallel}(\mathbf{x}_d) \right) \label{eq:reference_control_v} \\
& \hat{\omega} = k \, \text{atan} \left( \dfrac{\begin{bmatrix}
-\sin\varphi \\ \cos\varphi
\end{bmatrix}^\top \left( \mathbf{y} - \dfrac{\mathrm{\Pi}_{\mathcal{LF}(\mathbf{y}) \cap H_G}(\mathbf{x}_d)+\mathrm{\Pi}_{\mathcal{LF}(\mathbf{y})}(\mathbf{x}_d)}{2} \right)}{\begin{bmatrix}
\cos\varphi \\ \sin\varphi
\end{bmatrix}^\top \left( \mathbf{y} - \dfrac{\mathrm{\Pi}_{\mathcal{LF}(\mathbf{y}) \cap H_G}(\mathbf{x}_d)+\mathrm{\Pi}_{\mathcal{LF}(\mathbf{y})}(\mathbf{x}_d)}{2} \right)} \right) \label{eq:reference_control_omega}
\end{eqnarray}
\end{subequations}
with $k>0$ a fixed gain, $\mathcal{LF}(\mathbf{y}) \subset \mathcal{F}_{model}$ the convex polygon defining the local freespace at $\mathbf{y} = \mathbf{h}(\mathbf{x})$, and $H_\parallel$ and $H_G$ the lines defined in \cite{arslan_kod_WAFR2016} as 
\begin{align}
H_\parallel = & \left\{ \mathbf{z} \in \mathcal{F}_{model} \, \Big | \, \begin{bmatrix} -\sin\varphi \\ \cos\varphi \end{bmatrix}^\top (\mathbf{z}-\mathbf{y}) = 0 \right\} \\
H_G = & \left\{ \alpha \mathbf{y} + (1-\alpha) \mathbf{x}_d \in \mathcal{F}_{model} \, | \, \alpha \in \mathbb{R} \right\}
\end{align}

\subsubsection{Qualitative Properties}
The properties of the differential drive robot control law given in \eqref{eq:control_unicycle} can be summarized in the following theorem.

\begin{theorem} \label{theorem:control_se2}
The reactive controller for differential drive robots, given in \eqref{eq:control_unicycle}, leaves the freespace $\mathcal{F}_{map} \times S^1$ positively invariant, and its unique continuously differentiable flow, starting at almost any robot configuration $(\mathbf{x}, \psi) \in \mathcal{F}_{map} \times S^1$, asymptotically steers the robot to the goal location $\mathbf{x}^*$, without increasing $|| \mathbf{h}(\mathbf{x})-\mathbf{x}_d||$ along the way.
\end{theorem}

\begin{proof}
Included in Appendix \ref{appendix_sec:section_controller}.
\end{proof}
\section{Numerical Experiments}
\label{sec:simulations}
In this Section, we present numerical experiments that verify our formal results. All simulations were run in MATLAB using \texttt{ode45}, with control gain $k=0.4$ and $p=20$ for the R-function construction. The reader is also referred to our video attachment for a visualization of the examples presented here and more numerical simulations.

\begin{figure}[t]
\centering
\includegraphics[width=\textwidth]{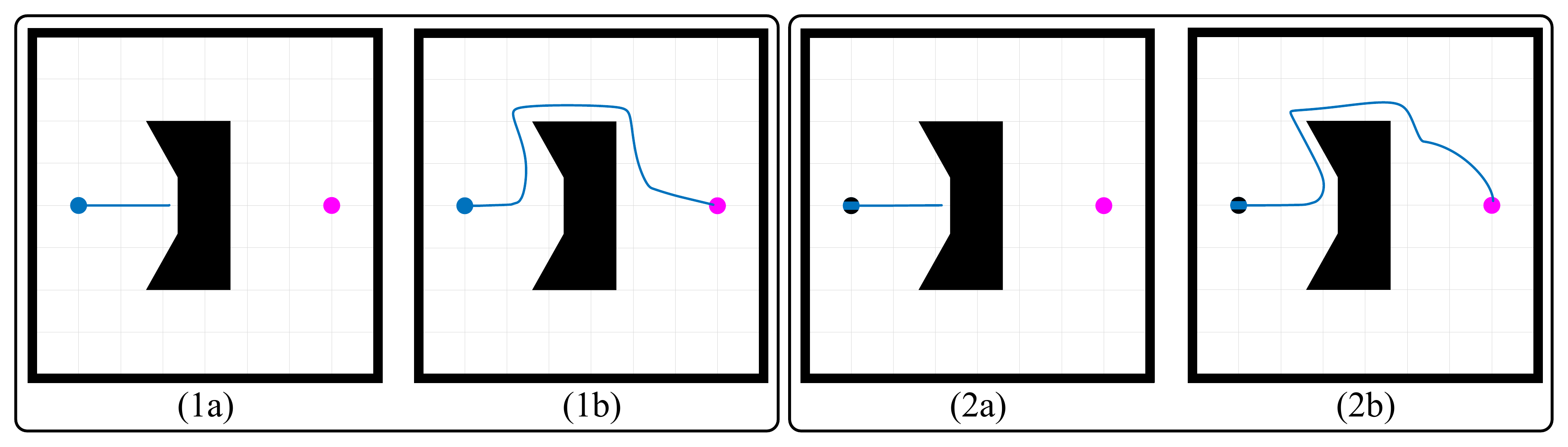}
\caption{Navigation around a U-shaped obstacle: 1) Fully actuated particle: (a) Original doubly reactive algorithm \cite{arslan_kod_WAFR2016}, (b) Our algorithm, 2) Differential drive robot: (a) Original doubly reactive algorithm \cite{arslan_kod_WAFR2016}, (b) Our algorithm.} \label{fig:ushaped}
\end{figure}

\begin{figure}[p]
\centering
\includegraphics[width=.8\textwidth]{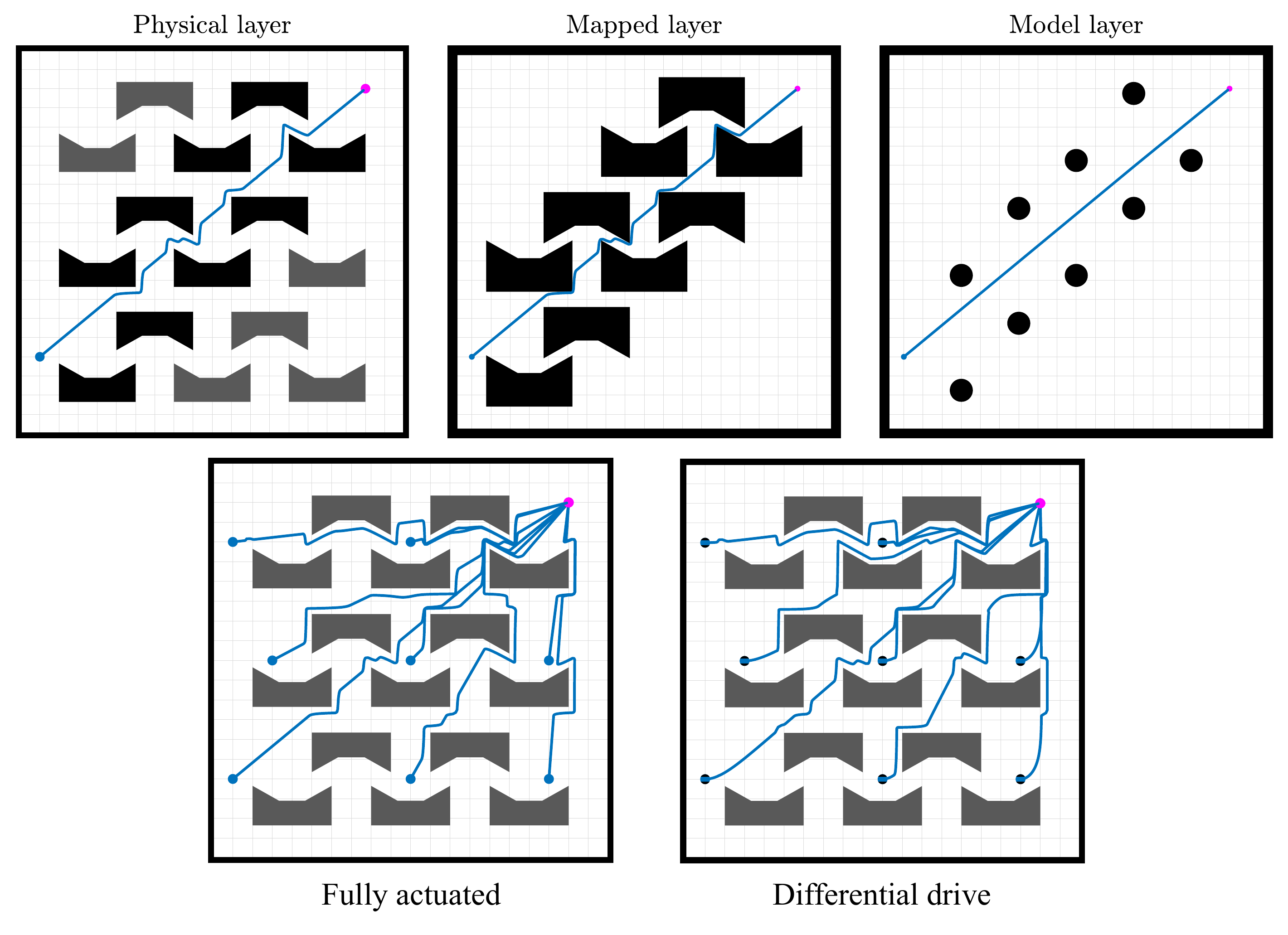}
\caption{Navigation in a cluttered environment with U-shaped obstacles. Top - Trajectories in the physical, mapped and model layers from a particular initial condition. Bottom - Convergence to the goal from several initial conditions: left - fully actuated robot, right - differential drive robot.} \label{fig:narrowpath}
\end{figure}

\begin{figure}[p]
\centering
\includegraphics[width=.8\textwidth]{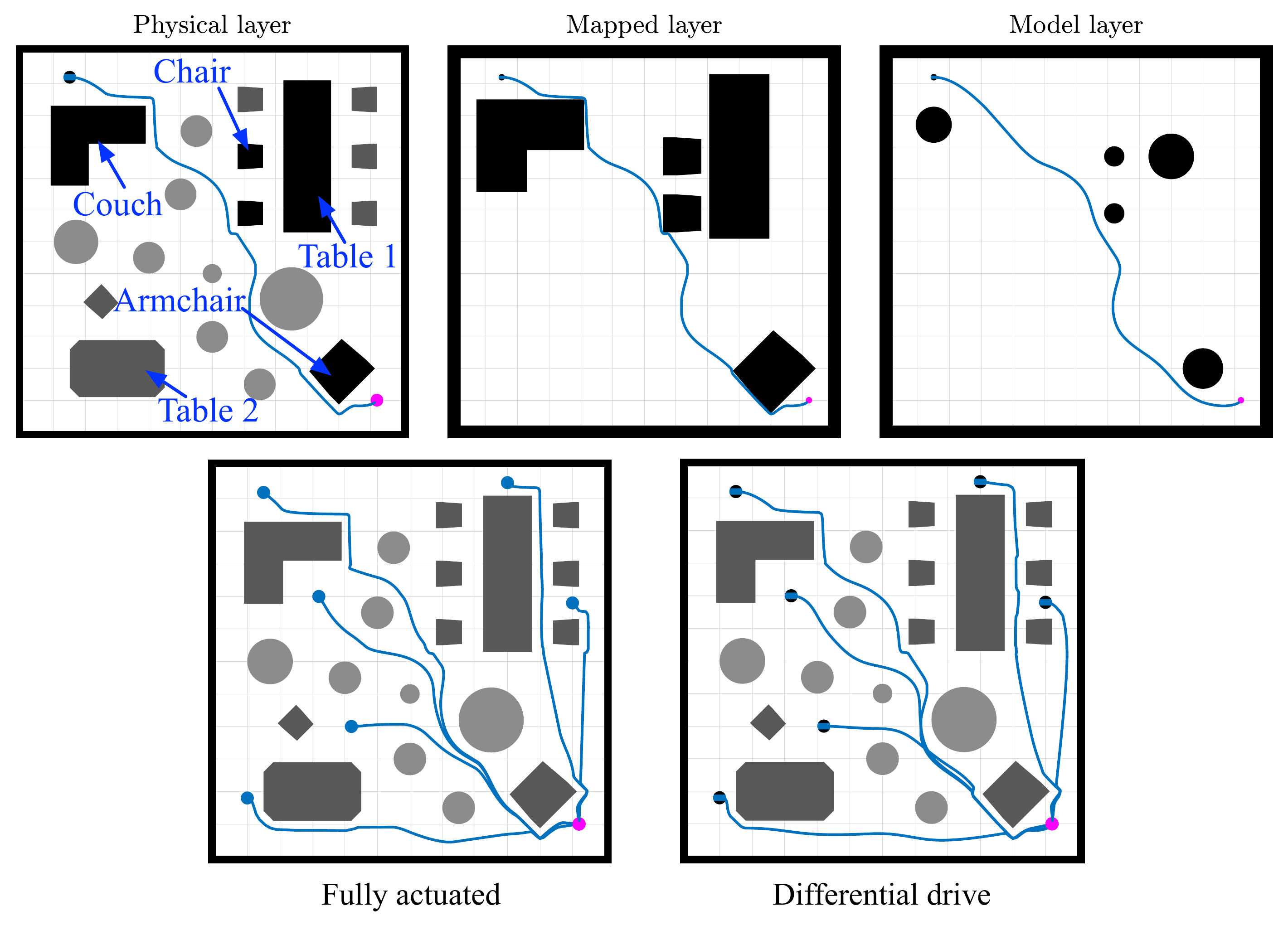}
\caption{Navigating a room cluttered with known star-shaped and unknown convex obstacles. Top - Trajectories in the physical, mapped and model layers from a particular initial condition. Bottom - Convergence to the goal from several initial conditions: left - fully actuated robot, right - differential drive robot. Mapped obstacles are shown in black, known obstacles in dark grey and unknown obstacles in light grey.} \label{fig:mixedstarconvex}
\end{figure}

\subsection{Comparison with Original Doubly Reactive Algorithm}
We begin with a comparison of our algorithm performance with the standalone version of the doubly reactive algorithm in \cite{arslan_kod_WAFR2016}, that we use in our construction. Fig. \ref{fig:ushaped} demonstrates the basic limitation of this algorithm; in the presence of a non-convex obstacle or a flat surface, whose curvature violates \cite[Assumption 2]{arslan_kod_WAFR2016}, the robot gets stuck in undesired local minima. On the contrary, our algorithm is capable of overcoming this limitation, on the premise that the robot can recognize the obstacle with star-shaped geometry at hand. The robot radius is $0.2$m and the value of $\varepsilon$ used for the obstacle is $0.3$.

\subsection{Navigation in a Cluttered Non-Convex Environment}
In the next set of numerical experiments, we evaluate the performance of our algorithm in a cluttered environment, packed with instances of the same U-shaped obstacle, with star-shaped geometry, we use in Fig. \ref{fig:ushaped}. Both the fully actuated and the differential drive robot are capable of converging to the desired goal from a variety of initial conditions, as shown in Fig. \ref{fig:narrowpath}. In the same figure, we also focus on a particular initial condition and include the trajectories observed in the physical, mapped and model layers. The robot radius is $0.25$m and value of $\varepsilon$ used for all the star-shaped obstacles in the environment is $0.3$.

\subsection{Navigation Among Mixed Star-Shaped and Convex Obstacles}
Finally, we report experiments in an environment cluttered with both star-shaped obstacles (with known geometry) and unknown convex obstacles. We consider a robot of radius $0.2$m navigating a room towards a goal. The robot can recognize familiar star-shaped obstacles (e.g., the couch, tables, armchair, chairs) but is unaware of several other convex obstacles in the environment. Fig. \ref{fig:mixedstarconvex} summarizes our results for several initial conditions. We also include trajectories observed in the physical, mapped and model layers during a single run. The value of $\varepsilon$ used for all the star-shaped obstacles in the environment is $0.3$.

\section{Conclusion and Future Work}
\label{sec:conclusion}
In this paper, we present a provably correct method for robot navigation in 2D environments cluttered with familiar but unexpected non-convex, star-shaped obstacles as well as completely unknown, convex obstacles. The robot uses a limited range onboard sensor, capable of recognizing, localizing and generating online from its catalogue of the familiar, non-convex shapes an implicit representation of each one. These sensory data and their interpreted representations underlie an online change of coordinates to a completely convex model planning space wherein a previously developed online construction yields a provably correct reactive controller that is pulled back to the physically sensed representation to generate the actual robot commands. Using a modified change of coordinates, the construction is also extended to differential drive robots, and numerical simulations further verify the validity of our formal results.

Experimental validation of our algorithm with deep learning techniques for object pose and triangular mesh recognition \cite{Pavlakos2017} is currently underway. Next steps target environments presenting geometry more complicated than star-shaped obstacles, by appropriately modifying the purging transformation algorithm for trees-of-stars, presented in \cite{rimon1992}. Future work aims to relax the required degree of partial knowledge and the separation assumptions needed for our formal results, by merging the ``implicit representation trees'' (e.g. see Fig. \ref{fig:rfunctions} in Appendix \ref{appendix:rfunctions}) online, when needed.

%
%
%
%
%
%
%
\bibliographystyle{splncs04}
\bibliography{references}

\appendix
\section{Proofs}
\label{appendix:proofs}

\subsection{Proofs of results in Section \ref{sec:geometric_transformation}}
\label{appendix_sec:section_geometric}

\begin{proof}[Proof of Lemma \ref{lemma:smooth_map}]
Since both the switches $\sigma_j$ and the deforming factors $\nu_j$ are smooth, for $j=1,\ldots,M$, the only technical challenge here is introduced by the fact that the number $M$ of discovered star-shaped obstacles in $\mathcal{F}_{map}$ is not constant and changes as the robot navigates the workspace.

Notice from \eqref{eq:zeta_derivative} that all the derivatives of $\eta_j$ used in the construction of the switch $\sigma_j$ for any $j$ are zero if and only if $\eta_j$ is zero. Therefore, in order to guarantee smoothness of $\mathbf{h}$, we just have to ensure that when a new obstacle $k$ is added to the semantic map, the value of $\sigma_k$ will be zero. This follows directly from the assumption that the sensor range $R$ is much greater than $\varepsilon_k$, which implies that when obstacle $k$ is discovered, the robot position $\mathbf{x}$ will lie outside the set $\{ \mathbf{q}\in\mathcal{F}_{map} \, | \, 0 \leq \beta_k(\mathbf{q}) < \varepsilon_k \}$ and therefore the value of $\sigma_k$ will be zero.
\end{proof}

\begin{proof}[Proof of Proposition \ref{proposition:diffeomorphism}]
First of all, the map $\mathbf{h}$ is smooth as shown in Lemma \ref{lemma:smooth_map}. Therefore, in order to prove that $\mathbf{h}$ is a $C^\infty$ diffeomorphism, we will follow the procedure outlined in \cite{massey1992}, also followed in \cite{Rimon_Koditschek_1989}, to show that
\begin{enumerate}
\item $\mathbf{h}$ has a non-singular differential on $\mathcal{F}_{map}$
\item $\mathbf{h}$ preserves boundaries, i.e., $\mathbf{h}(\partial_j\mathcal{F}_{map}) \subset \partial_j \mathcal{F}_{model}, j\in \{0,\ldots,M+N\}$.\footnote{Here we denote by $\partial_j \mathcal{F}$ the $j$-th connected component of the boundary of $\mathcal{F}$ (that corresponding to $\tilde{O}_j$), with $\partial_0\mathcal{F}$ the outer boundary of $\mathcal{F}$.}
\item the boundary components of $\mathcal{F}_{map}$ and $\mathcal{F}_{model}$ are pairwise homeomorphic, i.e. $\partial_j\mathcal{F}_{map} \cong \partial_j\mathcal{F}_{model}, j\in \{0,\ldots,M+N\}$.
\end{enumerate}

We begin with property 1. Using Lemma \ref{lemma:nonzero_switches} and observing from \eqref{eq:sigmaj} and \eqref{eq:grad_sigmaj} that a switch $\sigma_k, k \in \{1,\ldots,M\}$ is zero if and only if its gradient $\nabla \sigma_k$ is zero, we observe from \eqref{eq:map_differential} that $D_\mathbf{x}\mathbf{h}$ is either the identity map (which is non-singular) or depends only a single switch $\sigma_k, k \in \{1,\ldots,M\}$ when $0  \leq \beta_k(\mathbf{x})<\varepsilon_k$. In that case, we can isolate the $k$-th term in \eqref{eq:map_differential} and write the map differential as
\begin{align}
D_\mathbf{x}\mathbf{h} = D_\mathbf{x}\mathbf{h}_k = & \left[1+\sigma_k(\mathbf{x}) (\nu_k(\mathbf{x})-1)\right] \mathbf{I} + (\mathbf{x}-\mathbf{x}^*_k)  \left[\sigma_k(\mathbf{x})\nabla \nu_k(\mathbf{x})^\top \right. \nonumber \\
& \left.+ (\nu_k(\mathbf{x})-1) \nabla \sigma_k(\mathbf{x})^\top \right] \nonumber \\
=& \left[1+\sigma_k(\mathbf{x}) (\nu_k(\mathbf{x})-1)\right] \mathbf{I} + (\mathbf{x}-\mathbf{x}^*_k) \left[-\frac{\rho_k \sigma_k(\mathbf{x})}{||\mathbf{x}-\mathbf{x}^*_k||^3}(\mathbf{x}-\mathbf{x}^*_k)^\top\right. \nonumber \\ 
& \left. + \eta_k'(\beta_k(\mathbf{x}))(\nu_k(\mathbf{x})-1) \nabla \beta_k(\mathbf{x})^\top \right]
\end{align}
From this expression, we can find with some computation
\begin{align}
\text{tr}(D_\mathbf{x}\mathbf{h}_k) = & [1+\sigma_k(\mathbf{x})(\nu_k(\mathbf{x})-1)]+(1-\sigma_k(\mathbf{x})) \nonumber \\ 
& + \eta_k'(\beta_k(\mathbf{x}))(\nu_k(\mathbf{x})-1)(\mathbf{x}-\mathbf{x}^*_k)^\top \nabla \beta_k(\mathbf{x})
\end{align}
However, we know that
\begin{equation}
\frac{\sigma_k(\mathbf{x})-1}{\sigma_k(\mathbf{x})} \leq 0 < \nu_k(\mathbf{x})
\end{equation}
since $0<\sigma_k(\mathbf{x}) \leq 1$, giving $1+\sigma_k(\mathbf{x})(\nu_k(\mathbf{x})-1)>0$. Also, $\eta_k'(\beta_k(\mathbf{x})) < 0$ by construction (since $\beta_k(\mathbf{x})<\varepsilon_k$), $\nu_k(\mathbf{x})-1<0$ and $(\mathbf{x}-\mathbf{x}^*_k)^\top \nabla \beta_k(\mathbf{x}) > 0$ in the set $\{\mathbf{x} \in \mathcal{F}_{map} \, | \, 0 \leq \beta_k(\mathbf{x}) < \varepsilon_k \}$, because of Assumption \ref{assumption:epsilon}-(c). Therefore, we get $\text{tr}(D_\mathbf{x}\mathbf{h}_k)>0$ for all $\mathbf{x}$ such that $0 \leq \beta_k(\mathbf{x})<\varepsilon_k$. Also, since $\mathcal{F}_{map} \subset \mathbb{R}^2$, we can similarly compute 
\begin{align}
\text{det}(D_\mathbf{x}\mathbf{h}_k) = & \, g_k'(\beta_k(\mathbf{x}))(\nu_k(\mathbf{x})-1)[1+\sigma_k(\mathbf{x})(\nu_k(\mathbf{x})-1)](\mathbf{x}-\mathbf{x}^*_k)^\top \nabla \beta_k(\mathbf{x}) \nonumber \\
& + (1-\sigma_k(\mathbf{x}))[1+\sigma_k(\mathbf{x})(\nu_k(\mathbf{x})-1)]
\end{align}
which leads to $\text{det}(D_\mathbf{x}\mathbf{h}_k) > 0$ for all $\mathbf{x}$ such that $\beta_k(\mathbf{x})<\varepsilon_k$. Since $\text{det}(D_\mathbf{x}\mathbf{h}_k) > 0$ and $\text{tr}(D_\mathbf{x}\mathbf{h}_k)>0$, we conclude that $D_\mathbf{x}\mathbf{h}_k$ has two strictly positive eigenvalues in the set $\{\mathbf{x} \in \mathcal{F}_{map} \, | \, 0 \leq \beta_k(\mathbf{x}) < \varepsilon_k \}$. Since this is true for any $k \in \{ 1, \ldots, M \}$, it follows that $D_\mathbf{x}\mathbf{h}$ has two strictly positive eigenvalues in $\mathcal{F}_{map}$ and, thus, is non-singular in $\mathcal{F}_{map}$.

Next, pick a point $\mathbf{x} \in \partial_j\mathcal{F}_{map}$ for any $j \in \{0,\ldots,M+N\}$. This point could lie on the outer boundary of $\mathcal{F}_{map}$, on the boundary of one of the $N$ unknown but visible convex obstacles, or on the boundary of one of the $M$ star-shaped obstacles. In the first two cases, we have $\mathbf{h}(\mathbf{x})=\mathbf{x}$, while in the latter case
\begin{equation}
\mathbf{h}(\mathbf{x}) = \mathbf{x}^*_k+\frac{\rho_k}{||\mathbf{x}-\mathbf{x}^*_k||}(\mathbf{x}-\mathbf{x}^*_k)
\end{equation}
for some $k \in \{1,\ldots,M\}$, sending $\mathbf{x}$ to the boundary of the $k$-th disk in $\mathcal{F}_{model}$. This shows that we always have $\mathbf{h}(\mathbf{x}) \in \partial_j\mathcal{F}_{model}$ and, therefore, the map satisfies property 2.

Finally, property 3 derives from above and the fact that each boundary segment $\partial_j\mathcal{F}_{map}$ is an one-dimensional manifold, the boundary of either a convex set or a star-shaped set, both of which are homeomorphic to the corresponding boundary $\partial_j\mathcal{F}_{model}$.
\end{proof}

\subsection{Proofs of results in Section \ref{sec:reactive_controller}}
\label{appendix_sec:section_controller}

\begin{proof}[Proof of Proposition \ref{proposition:positive_invariance}]
Since $\mathbf{h}$ is just the identity transformation away from any star-shaped obstacle and the control law $\mathbf{u}$ guarantees collision avoidance in that case, as shown in \cite{arslan_kod_WAFR2016}, it suffices to show that the robot can never penetrate any star-shaped obstacle, i.e., for any $\mathbf{x}_c$ such that $\beta_k(\mathbf{x}_c)=0$ for some $k \in \{1,\ldots,M\}$, we have $\mathbf{u}(\mathbf{x}_c)^\top \nabla \beta_k(\mathbf{x}_c) \geq 0$. For such a point $\mathbf{x}_c$, we get from \eqref{eq:deforming_factors} and \eqref{eq:map_differential}
\begin{align}
D_{\mathbf{x}}\mathbf{h}(\mathbf{x}_c) = D_{\mathbf{x}}\mathbf{h}_k(\mathbf{x}_c) = & \left[1+\sigma_k(\mathbf{x}_c) (\nu_k(\mathbf{x}_c)-1)\right] \mathbf{I} + (\mathbf{x}_c-\mathbf{x}^*_k) \left[\sigma_k(\mathbf{x}_c) \nabla \nu_k(\mathbf{x}_c)^\top \nonumber \right. \\
& + \left.(\nu_k(\mathbf{x}_c)-1) \nabla \sigma_k(\mathbf{x}_c)^\top\right] \nonumber \\
=& \left[1+\sigma_k(\mathbf{x}_c) (\nu_k(\mathbf{x}_c)-1)\right] \mathbf{I} \nonumber \\
& + (\mathbf{x}_c-\mathbf{x}^*_k) \left[-\frac{\rho_k \sigma_k(\mathbf{x}_c)}{||\mathbf{x}_c-\mathbf{x}^*_k||^3} (\mathbf{x}_c-\mathbf{x}^*_k)^\top\right. \nonumber \\ 
&\left. + \eta_k'(\beta_k(\mathbf{x}_c))(\nu_k(\mathbf{x}_c)-1) \nabla \beta_k(\mathbf{x}_c)^\top \right] \nonumber \\
= & \frac{\rho_k}{||\mathbf{x}_c-\mathbf{x}^*_k||} \mathbf{I} + (\mathbf{x}_c-\mathbf{x}^*_k) \left[-\frac{\rho_k}{||\mathbf{x}_c-\mathbf{x}^*_k||^3}(\mathbf{x}_c-\mathbf{x}^*_k)^\top \right. \nonumber \\ 
& \left. + \eta_k'(\beta_k(\mathbf{x}_c))(\nu_k(\mathbf{x}_c)-1) \nabla \beta_k(\mathbf{x}_c)^\top \right]
\end{align}
since $\sigma_k(\mathbf{x}_c)=1$. Since $\mathcal{F}_{map} \subset \mathbb{R}^2$, we can explicitly compute the inverse of the 2x2 matrix $D_{\mathbf{x}}\mathbf{h}_k(\mathbf{x}_c)$ from its four elements $[D_{\mathbf{x}}\mathbf{h}_k(\mathbf{x}_c)]_{11}$, $[D_{\mathbf{x}}\mathbf{h}_k(\mathbf{x}_c)]_{12}$, $[D_{\mathbf{x}}\mathbf{h}_k(\mathbf{x}_c)]_{21}$, $[D_{\mathbf{x}}\mathbf{h}_k(\mathbf{x}_c)]_{22}$ as
\begin{equation}
D_{\mathbf{x}}\mathbf{h}_k(\mathbf{x}_c)^{-1} = \frac{1}{\text{det}(D_{\mathbf{x}}\mathbf{h}_k(\mathbf{x}_c)))} \begin{bmatrix}
[D_{\mathbf{x}}\mathbf{h}_k(\mathbf{x}_c)]_{22} & -[D_{\mathbf{x}}\mathbf{h}_k(\mathbf{x}_c)]_{12} \\ -[D_{\mathbf{x}}\mathbf{h}_k(\mathbf{x}_c)]_{21} & [D_{\mathbf{x}}\mathbf{h}_k(\mathbf{x}_c)]_{11}
\end{bmatrix}
\end{equation}
and after some simple computations, we can eventually find
\begin{equation}
[D_{\mathbf{x}}\mathbf{h}_k(\mathbf{x}_c)]^{-\top} \nabla \beta_k(\mathbf{x}_c) = \frac{\rho_k \, (\mathbf{x}_c-\mathbf{x}^*_k)^\top \nabla \beta_k(\mathbf{x}_c)}{||\mathbf{x}_c-\mathbf{x}^*_k||^3 \, \text{det}(D_{\mathbf{x}}\mathbf{h}_k(\mathbf{x}_c)))} (\mathbf{x}_c-\mathbf{x}^*_k) \label{eq:differential_gradient}
\end{equation}
On the other hand, 
\begin{equation}
\mathbf{u}(\mathbf{x}_c) = -k\, [D_{\mathbf{x}}\mathbf{h}_k(\mathbf{x}_c)]^{-1} \, \left(\mathbf{h}(\mathbf{x}_c) - \mathrm{\Pi}_{\mathcal{LF}(\mathbf{h}(\mathbf{x}_c))} (\mathbf{x}_d) \right)
\end{equation}
Since $\mathbf{x}_c$ belongs to the boundary of the obstacle $k$, then by construction of the diffeomorphism, $\mathbf{h}(\mathbf{x}_c)$ will belong to the boundary of the disk with radius $\rho_k$ centered at $\mathbf{x}^*_k$ and the associated hyperplane \cite{arslan_kod_WAFR2016} will be tangent to that disk at $\mathbf{h}(\mathbf{x}_c)$. Therefore, the projected goal $\mathrm{\Pi}_{\mathcal{LF}(\mathbf{h}(\mathbf{x}_c))} (\mathbf{x}_d)$ will belong to the halfspace defined by the outward normal vector from $\mathbf{x}^*_k$ to $\mathbf{h}(\mathbf{x}_c)$ at $\mathbf{h}(\mathbf{x}_c)$ and we have
\begin{equation}
\mathbf{u}(\mathbf{x}_c) =  [D_{\mathbf{x}}\mathbf{h}_k(\mathbf{x}_c)]^{-1} \, \mathbf{t}(\mathbf{x}_c)
\end{equation}
with $ \mathbf{t}(\mathbf{x}_c)^\top (\mathbf{h}(\mathbf{x}_c) - \mathbf{x}^*_k) \geq 0$. Since by construction of the diffeomorphism $\mathbf{h}(\mathbf{x}_c) = \mathbf{x}^*_k + \rho_k \frac{\mathbf{x}_c - \mathbf{x}^*_k}{||\mathbf{x}_c - \mathbf{x}^*_k||}$, we derive that
\begin{equation}
\mathbf{t}(\mathbf{x}_c)^\top (\mathbf{x}_c - \mathbf{x}^*_k) \geq 0 \label{eq:controlinnerproduct}
\end{equation}

Using the above results, we see that
\begin{align}
\mathbf{u}(\mathbf{x}_c)^\top \nabla \beta_k(\mathbf{x}_c) = & \left[\left[D_{\mathbf{x}}\mathbf{h}_k(\mathbf{x}_c)\right]^{-\top}\nabla \beta_k(\mathbf{x}_c) \right]^\top \mathbf{t}(\mathbf{x}_c) \nonumber \\
= & \frac{\rho_k \, (\mathbf{x}_c-\mathbf{x}^*_k)^\top \nabla \beta_k(\mathbf{x}_c)}{||\mathbf{x}_c-\mathbf{x}^*_k||^3 \, \text{det} \left(D_{\mathbf{x}}\mathbf{h}_k(\mathbf{x}_c)\right)} (\mathbf{x}_c-\mathbf{x}^*_k)^\top \mathbf{t}(\mathbf{x}_c) \geq 0
\end{align}
using \eqref{eq:controlinnerproduct} and the fact that $(\mathbf{x_c}-\mathbf{x}^*_k)^\top \nabla \beta_k(\mathbf{x}_c) >0$, since $\mathbf{x}_c$ belongs to the boundary of a star-shaped obstacle \cite{Rimon_Koditschek_1989}.
\end{proof}

\begin{proof}[Proof of Lemma \ref{lemma:equilibrium_stability}]
The proof of this lemma derives immediately from \cite[Propositions 5,11]{arslan_kod_WAFR2016}, from which we can infer that the set of stationary points of the vector field $D_\mathbf{x}\mathbf{h} \cdot \mathbf{u}(\mathbf{x})$, defined on $\mathcal{F}_{model}$, is $\{ \mathbf{x}_d\} \bigcup \{\mathbf{s}_j\}_{j \in \{1,\ldots,M\}} \bigcup_{i=1}^N \mathcal{G}_i$, with $\mathbf{x}_d$ being a locally stable equilibrium of $D_\mathbf{x}\mathbf{h} \cdot \mathbf{u}(\mathbf{x})$ and each other point being a nondegenerate saddle, since \cite[Assumption 2]{arslan_kod_WAFR2016} is satisfied for the obstacles in $\mathcal{F}_{model}$ by construction. To complete the proof, we just have to note that the index of an isolated zero of a vector field does not change under diffeomorphisms of the domain \cite{hirsch_1976}.
\end{proof}

\begin{proof}[Proof of Proposition \ref{proposition:attraction}]
Consider the smooth Lyapunov function candidate $V(\mathbf{x}) = ||\mathbf{h}(\mathbf{x})-\mathbf{x}_d||^2$, justified by the fact that $\mathbf{h}(\mathbf{x}_d) = \mathbf{x}_d$ by construction of the diffeomorphism, since we have assumed that $\beta_j(\mathbf{x}_d) > \varepsilon_j$ for all $j \in \{1,\ldots,M\}$. Using \eqref{eq:controller}
\begin{align}
\frac{dV}{dt} = & 2(\mathbf{h}(\mathbf{x})-\mathbf{x}_d)^\top(\mathbf{D}_\mathbf{x}\mathbf{h})\dot{\mathbf{x}} = -2k(\mathbf{h}(\mathbf{x})-\mathbf{x}_d)^\top \left(\mathbf{h}(\mathbf{x}) - \mathrm{\Pi}_{\mathcal{LF}(\mathbf{h}(\mathbf{x}))} (\mathbf{x}_d) \right) \nonumber \\
= & -2k\left(\mathbf{h}(\mathbf{x})-\mathrm{\Pi}_{\mathcal{LF}(\mathbf{h}(\mathbf{x}))} (\mathbf{x}_d)+\mathrm{\Pi}_{\mathcal{LF}(\mathbf{h}(\mathbf{x}))} (\mathbf{x}_d)-\mathbf{x}_d \right)^\top \left(\mathbf{h}(\mathbf{x}) - \mathrm{\Pi}_{\mathcal{LF}(\mathbf{h}(\mathbf{x}))} (\mathbf{x}_d) \right) \nonumber \\
= & -2k || \mathbf{h}(\mathbf{x})-\mathrm{\Pi}_{\mathcal{LF}(\mathbf{h}(\mathbf{x}))} (\mathbf{x}_d) ||^2 \nonumber \\
& +2k \left( \mathbf{x}_d - \mathrm{\Pi}_{\mathcal{LF}(\mathbf{h}(\mathbf{x}))} (\mathbf{x}_d) \right)^\top\left(\mathbf{h}(\mathbf{x}) - \mathrm{\Pi}_{\mathcal{LF}(\mathbf{h}(\mathbf{x}))} (\mathbf{x}_d) \right) \nonumber \\
\leq & -2k || \mathbf{h}(\mathbf{x})-\mathrm{\Pi}_{\mathcal{LF}(\mathbf{h}(\mathbf{x}))} (\mathbf{x}_d) ||^2 \leq 0
\end{align}
since $\mathbf{h}(\mathbf{x}) \in \mathcal{LF}(\mathbf{h}(\mathbf{x}))$, which implies that 
\begin{equation}
\left( \mathbf{x}_d - \mathrm{\Pi}_{\mathcal{LF}(\mathbf{h}(\mathbf{x}))} (\mathbf{x}_d) \right)^\top\left(\mathbf{h}(\mathbf{x}) - \mathrm{\Pi}_{\mathcal{LF}(\mathbf{h}(\mathbf{x}))} (\mathbf{x}_d) \right) \leq 0
\end{equation}
since either $\mathbf{x}_d = \mathrm{\Pi}_{\mathcal{LF}(\mathbf{h}(\mathbf{x}))} (\mathbf{x}_d)$, or $\mathbf{x}_d$ and $\mathbf{h}(\mathbf{x})$ are separated by a hyperplane passing through $\mathrm{\Pi}_{\mathcal{LF}(\mathbf{h}(\mathbf{x}))} (\mathbf{x}_d)$. Therefore, similarly to \cite{arslan_kod_WAFR2016}, using LaSalle's invariance principle we see that every trajectory starting in $\mathcal{F}_{map}$ approaches the largest invariant set in $\{ \mathbf{x} \in \mathcal{F}_{map} \, | \, \dot{V}(\mathbf{x})=0 \}$, i.e. the equilibrium points of \eqref{eq:controller}. The desired result follows from Lemma \ref{lemma:equilibrium_stability}, since $\mathbf{x}_d$ is the only locally stable equilibrium of our control law and the rest of the stationary points are nondegenerate saddles, whose regions of attraction have empty interior in $\mathcal{F}_{map}$.
\end{proof}

\begin{proof}[Proof of Proposition \ref{proposition:diffeo_se2}]
Note that the jacobian of $\overline{\mathbf{h}}$ will be given by
\begin{equation}
D_{\overline{\mathbf{x}}} \overline{\mathbf{h}} = \begin{bmatrix}
D_\mathbf{x}\mathbf{h} & \vline & \mathbf{0}_{2 \times 1} \\ \hline D_\mathbf{x}\xi & \vline & \dfrac{\partial \xi}{\partial \psi}
\end{bmatrix} \label{eq:differential_se2}
\end{equation}
Since we have already shown in Proposition \ref{proposition:diffeomorphism} that $D_\mathbf{x}\mathbf{h}$ is non-singular, it suffices to show that $\frac{\partial \xi}{\partial \psi} \neq 0$ for all $\overline{\mathbf{x}} \in \mathcal{F}_{map} \times S^1$. From \eqref{eq:phi_definition} we can derive
\begin{equation}
\frac{\partial \xi}{\partial \psi} = \frac{\text{det}(D_\mathbf{x}\mathbf{h})}{||\mathbf{e}(\overline{\mathbf{x}})||^2}
\end{equation}
Therefore, we immediately get that $\frac{\partial \xi}{\partial \psi} \neq 0$ for all $\overline{\mathbf{x}} \in \mathcal{F}_{map} \times S^1$ since $\text{det}(D_\mathbf{x}\mathbf{h}) \neq 0$ and $||\mathbf{e}(\overline{\mathbf{x}})|| \neq 0$ for all $\mathbf{x} \in \mathcal{F}_{map}$, because $D_\mathbf{x}\mathbf{h}$ is non-singular on $\mathcal{F}_{map}$. This implies that $D_{\overline{\mathbf{x}}}\overline{\mathbf{h}}$ is non-singular on $\mathcal{F}_{map} \times S^1$.

Next, we note that $\partial\left(\mathcal{F}_{map} \times S^1 \right) = \partial \mathcal{F}_{map} \times S^1$, since $S^1$ is a manifold without boundary. Similarly, $\partial\left(\mathcal{F}_{model} \times S^1 \right) = \partial \mathcal{F}_{model} \times S^1$. Hence, we can easily complete the proof following a similar procedure with the end of the proof of Proposition \ref{proposition:diffeomorphism}.
\end{proof}

\begin{proof}[Proof of Theorem \ref{theorem:control_se2}]
We have already established that $||\mathbf{e}(\overline{\mathbf{x}})||$ and $\frac{\partial \xi}{\partial \psi}$ are nonzero for all $\overline{\mathbf{x}} \in \mathcal{F}_{map} \times S^1$ in the proof of Proposition \ref{proposition:diffeo_se2}, which implies that $v$ and $\omega$ can have no singular points. Also notice that $||\mathbf{e}(\overline{\mathbf{x}})||$, $\frac{\partial \xi}{\partial \psi}$ and $D_\mathbf{x}\xi\begin{bmatrix}
\cos\psi & \sin\psi
\end{bmatrix}^\top$ are all smooth. Hence, the uniqueness and existence of the flow generated by control law \eqref{eq:control_unicycle} can be established similarly to \cite{arslan_kod_WAFR2016} through the flow properties of the controller in \cite{astolfi_1999} (that we use here in \eqref{eq:control_unicycle_reference}) and the facts that metric projections onto moving convex cells are piecewise continuously differentiable \cite{Kuntz-1994,shapiro-1988} and the composition of piecewise continuously differentiable functions is piecewise continuously differentiable and, therefore, locally Lipschitz \cite{chaney-1990}.

Positive invariance of $\mathcal{F}_{map} \times S^1$ can be proven following similar patterns with the proof of Proposition \ref{proposition:positive_invariance}. Namely, it suffices to show that the robot can never penetrate an obstacle, i.e., for any placement $(\mathbf{x}_c,\psi_c)$ such that $\beta_k(\mathbf{x}_c)=0$ for some index $k \in \{1,\ldots,M\}$, we definitely have 
\begin{equation}
\nabla \beta_k(\mathbf{x}_c)^\top \begin{bmatrix}
v_c \cos\psi_c \\ v_c \sin\psi_c
\end{bmatrix}\geq 0
\end{equation}
for any $\psi_c \in S^1$. We know from \eqref{eq:vector_field_se2} that
\begin{equation}
\begin{bmatrix}
v_c \cos\psi_c \\ v_c \sin\psi_c
\end{bmatrix} = [D_{\mathbf{x}}\mathbf{h}(\mathbf{x}_c)]^{-1} \begin{bmatrix}
\hat{v}_c \cos\varphi_c \\ \hat{v}_c \sin\varphi_c
\end{bmatrix}
\end{equation}
Therefore
\begin{align}
\nabla \beta_k(\mathbf{x}_c)^\top \begin{bmatrix}
v_c \cos\psi_c \\ v_c \sin\psi_c
\end{bmatrix} = & \nabla \beta_k(\mathbf{x}_c)^\top \left([D_{\mathbf{x}}\mathbf{h}(\mathbf{x}_c)]^{-1} \begin{bmatrix}
\hat{v}_c \cos\varphi_c \\ \hat{v}_c \sin\varphi_c
\end{bmatrix}\right) \nonumber \\
= & \left([D_{\mathbf{x}}\mathbf{h}(\mathbf{x}_c)]^{-\top} \nabla \beta_k(\mathbf{x}_c) \right)^\top \begin{bmatrix}
\hat{v}_c \cos\varphi_c \\ \hat{v}_c \sin\varphi_c
\end{bmatrix} \nonumber \\
= & \frac{\rho_k \, (\mathbf{x}_c-\mathbf{x}^*_k)^\top \nabla \beta_k(\mathbf{x}_c)}{||\mathbf{x}_c-\mathbf{x}^*_k||^3 \text{det}(D_{\mathbf{x}}\mathbf{h}(\mathbf{x}_c))} (\mathbf{x}_c-\mathbf{x}^*_k)^\top \begin{bmatrix}
\hat{v}_c \cos\varphi_c \\ \hat{v}_c \sin\varphi_c
\end{bmatrix}
\end{align}
using \eqref{eq:differential_gradient}. Hence, using the results from Proposition \ref{proposition:positive_invariance}, we see that positive invariance of $\mathcal{F}_{map} \times S^1$ under law \eqref{eq:control_unicycle} is equivalent to positive invariance of $\mathcal{F}_{model} \times S^1$ under law \eqref{eq:control_unicycle_reference}, which is guaranteed from \cite[Proposition 12]{arslan_kod_WAFR2016}.

Finally, consider the smooth Lyapunov function candidate $V(\mathbf{x}) = ||\mathbf{h}(\mathbf{x})-\mathbf{x}_d||^2$. Then
\begin{align}
\frac{dV}{dt} = & 2(\mathbf{h}(\mathbf{x})-\mathbf{x}_d)^\top(\mathbf{D}_\mathbf{x}\mathbf{h})\dot{\mathbf{x}} \nonumber \\
= &  2 v\, (\mathbf{h}(\mathbf{x})-\mathbf{x}_d)^\top(\mathbf{D}_\mathbf{x}\mathbf{h}) \begin{bmatrix}
\cos\psi \\ \sin\psi
\end{bmatrix} \nonumber \\
= & 2 \overline{v} (\mathbf{h}(\mathbf{x})-\mathbf{x}_d)^\top \begin{bmatrix} \cos\xi(\overline{\mathbf{x}}) \\ \sin\xi(\overline{\mathbf{x}})
\end{bmatrix} \nonumber \\
= & -2k (\mathbf{h}(\mathbf{x})-\mathbf{x}_d)^\top \begin{bmatrix} \cos\xi(\overline{\mathbf{x}}) \\ \sin\xi(\overline{\mathbf{x}})
\end{bmatrix} \begin{bmatrix} \cos\xi(\overline{\mathbf{x}}) \\ \sin\xi(\overline{\mathbf{x}})
\end{bmatrix}^\top \left(\mathbf{h}(\mathbf{x})-\mathrm{\Pi}_{\mathcal{LF}(\mathbf{h}(\mathbf{x})) \cap H_\parallel}(\mathbf{x}_d) \right) \nonumber \\
= & -2k (\mathbf{h}(\mathbf{x})-\mathbf{x}_d)^\top \left(\mathbf{h}(\mathbf{x})-\mathrm{\Pi}_{\mathcal{LF}(\mathbf{h}(\mathbf{x})) \cap H_\parallel}(\mathbf{x}_d) \right) \nonumber
\end{align}
since $\begin{bmatrix} \cos\xi(\overline{\mathbf{x}}) \\ \sin\xi(\overline{\mathbf{x}})
\end{bmatrix} \begin{bmatrix} \cos\xi(\overline{\mathbf{x}}) \\ \sin\xi(\overline{\mathbf{x}})
\end{bmatrix}^\top$ is just the projection operator on the line defined by the vector $\begin{bmatrix} \cos\xi(\overline{\mathbf{x}}) & \sin\xi(\overline{\mathbf{x}})
\end{bmatrix}^\top$, with which $\left(\mathbf{h}(\mathbf{x})-\mathrm{\Pi}_{\mathcal{LF}(\mathbf{h}(\mathbf{x})) \cap H_\parallel}(\mathbf{x}_d) \right)$ is already parallel. Following this result, we get
\begin{equation}
\frac{dV}{dt} \leq -2k \Big|\Big|\mathbf{h}(\mathbf{x})-\mathrm{\Pi}_{\mathcal{LF}(\mathbf{h}(\mathbf{x})) \cap H_\parallel}(\mathbf{x}_d)\Big|\Big|^2 \leq 0
\end{equation}
since, similarly to the proof of Proposition \ref{proposition:attraction}, we have
\begin{equation}
\left(\mathbf{x}_d - \mathrm{\Pi}_{\mathcal{LF}(\mathbf{h}(\mathbf{x})) \cap H_\parallel}(\mathbf{x}_d) \right)^\top \left(\mathbf{h}(\mathbf{x})-\mathrm{\Pi}_{\mathcal{LF}(\mathbf{h}(\mathbf{x})) \cap H_\parallel}(\mathbf{x}_d) \right) \leq 0
\end{equation}
Therefore, using LaSalle's invariance principle, we see that every trajectory starting in $\mathcal{F}_{map} \times S^1$ approaches the largest invariant set in $\{ (\mathbf{x},\psi) \in \mathcal{F}_{map} \times S^1 \, | \, \dot{V}(\mathbf{x})=0 \} = \{ (\mathbf{x},\psi) \in \mathcal{F}_{map} \times S^1 \, | \, \mathbf{h}(\mathbf{x}) = \mathrm{\Pi}_{\mathcal{LF}(\mathbf{h}(\mathbf{x})) \cap H_\parallel}(\mathbf{x}_d) \}$. At the same time, we know from \eqref{eq:control_unicycle_reference} that $\mathbf{h}(\mathbf{x}) = \mathrm{\Pi}_{\mathcal{LF}(\mathbf{h}(\mathbf{x})) \cap H_\parallel}(\mathbf{x}_d)$ implies $v=0$. From \eqref{eq:control_unicycle}, for $v=0$, we get that $\omega$ will be zero at points where $\hat{\omega}$ is zero, i.e. at points $(\mathbf{x},\psi) \in \mathcal{F}_{map} \times S^1$ where
\begin{equation}
\begin{bmatrix}
-\sin\xi(\overline{\mathbf{x}}) \\ \cos\xi(\overline{\mathbf{x}})
\end{bmatrix}^\top \left( \mathbf{h}(\mathbf{x}) - \dfrac{\mathrm{\Pi}_{\mathcal{LF}(\mathbf{h}(\mathbf{x})) \cap H_G}(\mathbf{x}_d)+\mathrm{\Pi}_{\mathcal{LF}(\mathbf{h}(\mathbf{x}))}(\mathbf{x}_d)}{2} \right) = 0
\end{equation}
Therefore the largest invariant set in $\{ (\mathbf{x},\psi) \, | \, \mathbf{h}(\mathbf{x}) = \mathrm{\Pi}_{\mathcal{LF}(\mathbf{h}(\mathbf{x})) \cap H_\parallel}(\mathbf{x}_d) \}$ is the set of points $\overline{\mathbf{x}} = (\mathbf{x},\psi)$ where the following two conditions are satisfied
\begin{align}
&  \mathbf{h}(\mathbf{x}) = \mathrm{\Pi}_{\mathcal{LF}(\mathbf{h}(\mathbf{x})) \cap H_\parallel}(\mathbf{x}_d) \\
& \begin{bmatrix}
-\sin\xi(\overline{\mathbf{x}}) \\ \cos\xi(\overline{\mathbf{x}})
\end{bmatrix}^\top \left( \mathbf{h}(\mathbf{x}) - \dfrac{\mathrm{\Pi}_{\mathcal{LF}(\mathbf{h}(\mathbf{x})) \cap H_G}(\mathbf{x}_d)+\mathrm{\Pi}_{\mathcal{LF}(\mathbf{h}(\mathbf{x}))}(\mathbf{x}_d)}{2} \right) = 0
\end{align}
Using a similar argument to \cite[Proposition 12]{arslan_kod_WAFR2016}, we can, therefore, verify that the set of stationary points of law \eqref{eq:control_unicycle} is given by 
\begin{align}
& \{\mathbf{x}_d\} \times (-\pi,\pi] \nonumber \\
&\bigcup \left\{(\mathbf{q},\psi) \, \Big | \, \mathbf{q} \in \{\mathbf{h}^{-1}(\mathbf{s}_j)\}_{j \in \{1,\ldots,M\}} \bigcup_{i=1}^N \mathcal{G}_i, \begin{bmatrix}
-\sin\xi(\mathbf{q},\psi) \\ \cos\xi(\mathbf{q},\psi)
\end{bmatrix}^\top(\mathbf{q}-\mathbf{x}_d) = 0 \right \} 
\end{align}
using \eqref{eq:saddles}. We can then invoke a similar argument to Proposition \ref{proposition:attraction} to show that $\mathbf{x}_d$ locally attracts with any orientation $\psi$, while any configuration associated with any other equilibrium point is a nondegenerate saddle whose stable manifold is a set of measure zero, and the result follows.
\end{proof}
\section{Implicit Representation of Obstacles With R-functions}
\label{appendix:rfunctions}

In this work, looking ahead toward handling in a more modular fashion the general class of obstacle shapes encompassed by the star-tree methods from the traditional navigation function literature \cite{Rimon_Koditschek_1989,rimon1992}, we depart from individuated homogeneous implicit function representation of our memorized catalogue elements in favor of the R-function compositions \cite{Rvachev_1963}, explored by Rimon \cite{Rimon_1990} and explicated within the field of constructive solid geometry by Shapiro \cite{shapiro2007}. We believe that this modular representation of shape will be helpful in the effort now in progress to instantiate the posited mapping oracle for obstacles with known geometry, whose triangular mesh can be identified in real time using state-of-the-art techniques \cite{Kong_Lin_Lucey_2017,Kar_Tulsiani_Carreira_Malik_2015,Pavlakos2017} in order to extract implicit function representations for polygonal obstacles.

\subsection{Preliminary Definitions}
We begin by providing a definition of an R-function \cite{shapiro2007}.
\begin{definition}
A function $\gamma_\Phi:\mathbb{R}^n \rightarrow \mathbb{R}$ is an R-function if there exists a (binary) logic function $\Phi:\mathbb{B} \rightarrow \mathbb{B}$, called the companion function, that satisfies the relation
\begin{equation}
\Phi(S_2(w_1),\ldots,S_2(w_n)) = S_2(\gamma_\Phi(w_1,\ldots,w_n))
\end{equation}
with $(w_1,\ldots,w_n) \in \mathbb{R}^n$ and $S_2$ the Heaviside characteristic function $S_2:\mathbb{R} \rightarrow \mathbb{B}$ of the interval $[0+,\infty)$ defined as\footnote{In \cite{shapiro2007}, it is assumed that zero is always signed: either $+0$ or $-0$, which allows the authors to determine membership of zero either to the set of positive or to the set of negative numbers. This strange assumption is employed to resolve pathological cases, where the membership of zero causes R-function discontinuities and is not of particular importance in our setting.}
\begin{equation}
S_2(\chi) = \left\{ \begin{matrix}
0, \quad \chi \leq -0 \\ 1, \quad \chi \geq +0
\end{matrix} \right.
\end{equation}
\end{definition}
Informally, a real function $\gamma_\Phi$ is an R-function if it can change its property (sign) only when some of its arguments change the same property (sign) \cite{shapiro2007}. For example, the companion logic function for the R-function $\gamma(x,y)=x y$ is $X \Leftrightarrow Y$; we just check that $S_2(xy) = (S_2(x) \Leftrightarrow S_2(y))$.

In this work, we use the following (symbolically written) R-functions \cite{shapiro2007}
\begin{align}
\neg x & :=-x \label{eq:rfunction_negation}\\
x_1 \wedge x_2 & := x_1+x_2-\left(x_1^p+x_2^p\right)^\frac{1}{p} \label{eq:rfunction_conjunction} \\
x_1 \vee x_2 & := x_1+x_2+\left(x_1^p+x_2^p\right)^\frac{1}{p} \label{eq:rfunction_disjunction}
\end{align}
with companion logic functions the logical negation $\neg$, conjunction $\wedge$ and disjunction $\vee$ respectively and $p$ a positive integer. Intuitively, the author in \cite{shapiro2007} uses the triangle inequality with the $L_p$-norm to derive R-functions with specific properties. 

\subsection{Description of the Algorithm}
\label{subsec:rfunction_algorithm}
R-functions have several interesting properties but, most importantly, provide machinery to construct implicit representations for sets built from other, primitive sets. Namely, in order to obtain a real function inequality $\gamma \geq 0$ defining a set $\Omega$ constructed from primitive sets $\Omega_j$, it suffices to construct an appropriate R-function and substitute for its arguments the real functions $\omega_j$ defining the primitive sets $\Omega_j$ implicitly as $\omega_j \geq 0$ \cite[Theorem 3]{shapiro2007}. In our case, the set $\Omega$ would be the star-shaped polygon $\tilde{O}_i^*$ we want to represent, the sets $\Omega_j$ would be half-spaces induced by the polygon edges and the functions $\omega_j:\mathbb{R}^2 \rightarrow \mathbb{R}$ their corresponding hyperplane equations, which are of the form
\begin{equation}
\omega_j(\mathbf{x}) = (\mathbf{x}-\mathbf{x}_{c,j})^\top \mathbf{n}_j \label{eq:hyperplane}
\end{equation}
Here $\mathbf{x}_{c,j}$ is any arbitrary point on the edge hyperplane and $\mathbf{n}_j$ its normal vector, pointing towards the interior of the polygon.


This result allows us to use a variant of the method presented in \cite{shapiro2007} and construct representations of polygons in the form of AND-OR trees \cite{russell_2009}, as shown in the example of Fig. \ref{fig:rfunctions}. Briefly, the interior of a polygon can be represented as the intersection of two or more {\it polygonal chains}, i.e. sequences of edges that meet at the polygon's convex hull. In the same way, each of these chains can then be split recursively into smaller subchains at the vertices of its convex hull to form a tree structure. The root node of the tree is the original polygon, with each other node corresponding to a polygonal chain; the leaves of the tree are single hyperplanes, the edges of the polygon described by functions $\omega_j$. If the split occurs at a concave vertex of the {\it original} polygon, then the subchains are combined using set union (i.e. disjunction); otherwise, they are combined using set intersection (i.e. conjunction), as shown in Fig. \ref{fig:rfunctions}. In this way, by having as input just the vertices of the polygon in counterclockwise order, we are able to construct an implicit representation for each node of the tree bottom-up, using the R-functions \eqref{eq:rfunction_conjunction} and \eqref{eq:rfunction_disjunction}, until we reach the root node of the tree. Since, for our application, we want $\beta_i > 0$ in the exterior of the obstacle $\tilde{O}_i^*$, we negate the result (i.e., we use the R-function \eqref{eq:rfunction_negation}) to obtain the obstacle function $\beta_i$, which is analytic everywhere except for the polygon vertices \cite{shapiro2007}. This implies that our results in Section \ref{sec:geometric_transformation} still hold, with the map $\mathbf{h}$ being a $C^\infty$ diffeomorphism away from the polygon vertices.

We apply the algorithm described above to extract (offline) the AND-OR tree associated with the implicit function $\beta_{0i}$, describing the obstacle $\tilde{O}_i^*$ with its center located at the origin, as shown in Fig. \ref{fig:rfunctions}. Then, when our sensor recognizes $\tilde{O}_i^*$ online and identifies its pose as a rotation $\mathbf{R}_i$ on the plane followed by a translation $\mathbf{x}_i^*$ of its center, we simply find the value of $\beta_i(\mathbf{x})$ at a point $\mathbf{x}$ as $\beta_i(\mathbf{x}) = \beta_{0i}\left(\mathbf{R}_i^\top(\mathbf{x}-\mathbf{x}_i^*)\right)$, by simply invoking the inverse homogeneous transformation that takes the center of $\tilde{O}_i^*$ back to the origin. This allows for efficient online computation of $\beta_i(\mathbf{x})$ as the robot navigates its environment.

\begin{figure}[t]
\centering
\includegraphics[width=\textwidth]{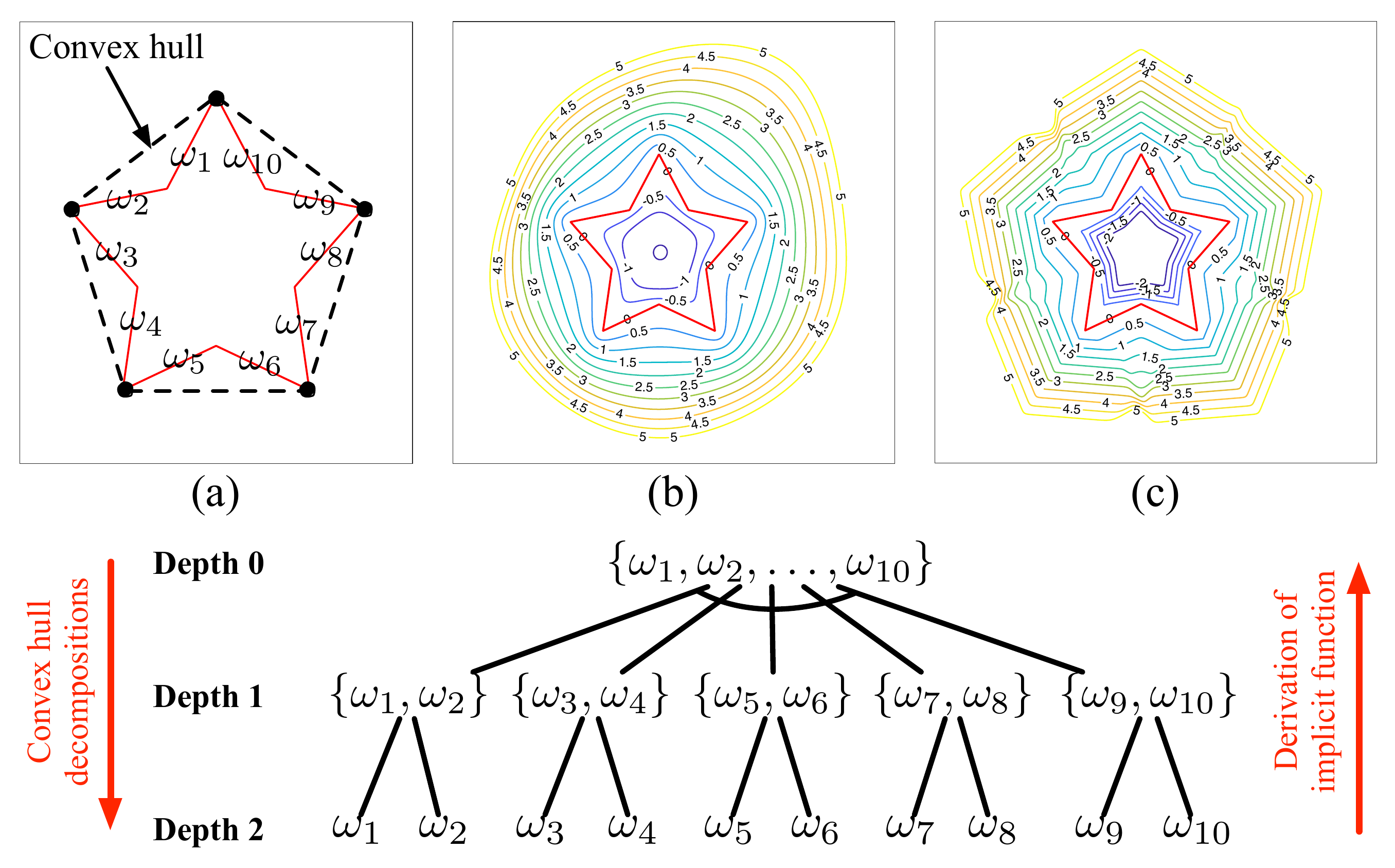}
\caption{Top: (a) An example of a star-shaped polygonal obstacle and the corresponding $\omega_j$ functions, (b) Level curves of the corresponding implicit function $\beta$ for $p=2$, (c) Level curves of the corresponding implicit function $\beta$ for $p=20$, Bottom: The AND-OR tree, constructed by the algorithm described in Section \ref{subsec:rfunction_algorithm} to represent this polygon. The polygon is split at the vertices of the convex hull to generate five subchains at depth 1. Each of these subchains is then split into two subchains at depth 2. The subchains at depth 2 (1) are combined via disjunction (conjunction), since they meet at non-convex (convex) vertices of the original polygon. Following this procedure, we get our implicit function $\beta = \neg \left( (\omega_1 \vee \omega_2) \wedge (\omega_3 \vee \omega_4) \wedge (\omega_5 \vee \omega_6) \wedge (\omega_7 \vee \omega_8) \wedge (\omega_9 \vee \omega_{10}) \right)$.} \label{fig:rfunctions}
\end{figure}

\subsection{R-functions as Approximations of the Distance Function}
It is important to mention that, away from the corners and in a neighborhood of the polygon, {\it normalized} R-functions constructed using \eqref{eq:rfunction_negation}-\eqref{eq:rfunction_disjunction} behave as smooth $p$-th order approximations of the (non-differentiable) distance function to the polygon, as shown in Fig. \ref{fig:rfunctions}-(b),(c). The reader is referred to \cite{shapiro2007} for more details; in our setting, a sufficient condition for normalization is to make sure that for each $\omega_j$ given in \eqref{eq:hyperplane}, the corresponding normal vector $\mathbf{n}_j$ has unit norm \cite{shapiro2007}.

This property is quite useful for our purposes, as it endows the implicit representation of our polygons with a physical meaning; this facilitates the choice of $\varepsilon_i$ for each known star-shaped obstacle $i$, compared to other representations (e.g. the homogeneous function representations in \cite{Rimon_Koditschek_1989}). Finally, the distance-like behavior of the R-functions guarantees that Assumption \ref{assumption:epsilon}-(c) is satisfied for each star-shaped obstacle, as demonstrated in Fig. \ref{fig:rfunctions}. Numerical experimentation showed that even $p=2$ gives sufficiently good results in our setting.

\section{Calculation of $D_\mathbf{x}\xi$}
\label{appendix:calculation_jacobian}

We can calculate
\begin{align}
\frac{\partial [D_\mathbf{x}\mathbf{h}]_{11}}{\partial x} = & \sum_{j=1}^M \left[ 2\sigma_j \frac{\partial \nu_j}{\partial x} + 2(\nu_j-1)\frac{\partial \sigma_j}{\partial x} + 2(x-x^*_j)\frac{\partial \sigma_j}{\partial x}\frac{\partial \nu_j}{\partial x} \right. \nonumber \\
& \left. + (x-x^*_j)\sigma_j\frac{\partial^2 \nu_j}{\partial x^2} + (x-x^*_j)(\nu_j-1)\frac{\partial^2 \sigma_j}{\partial x^2} \right] \\
\frac{\partial [D_\mathbf{x}\mathbf{h}]_{11}}{\partial y} = & \sum_{j=1}^M \left[ \sigma_j \frac{\partial \nu_j}{\partial y} + (\nu_j-1)\frac{\partial \sigma_j}{\partial y} + (x-x^*_j)\frac{\partial \sigma_j}{\partial y}\frac{\partial \nu_j}{\partial x} + (x-x^*_j)\sigma_j\frac{\partial^2 \nu_j}{\partial x \partial y} \right. \nonumber \\
& \left. + (x-x^*_j) \frac{\partial \sigma_j}{\partial x} \frac{\partial \nu_j}{\partial y} +(x-x^*_j)(\nu_j-1)\frac{\partial^2 \sigma_j}{\partial x \partial y} \right] \\
\frac{\partial [D_\mathbf{x}\mathbf{h}]_{12}}{\partial x} = & \sum_{j=1}^M \left[ \sigma_j \frac{\partial \nu_j}{\partial y} + (x-x^*_j)\frac{\partial \sigma_j}{\partial x} \frac{\partial \nu_j}{\partial y} + (x-x^*_j)\sigma_j\frac{\partial^2 \nu_j}{\partial x \partial y} \right. \nonumber \\
& \left. + (\nu_j-1)\frac{\partial \sigma_j}{\partial y} + (x-x^*_j) \frac{\partial \sigma_j}{\partial y} \frac{\partial \nu_j}{\partial x} + (x-x^*_j)(\nu_j-1)\frac{\partial^2\sigma_j}{\partial x \partial y} \right] \\
\frac{\partial [D_\mathbf{x}\mathbf{h}]_{12}}{\partial y} = & \sum_{j=1}^M \left[2(x-x^*_j)\frac{\partial \sigma_j}{\partial y} \frac{\partial \nu_j}{\partial y} + (x-x^*_j)\sigma_j\frac{\partial^2 \nu_j}{\partial y^2} \right. \nonumber \\
& \left. + (x-x^*_j)(\nu_j-1)\frac{\partial^2 \sigma_j}{\partial y^2} \right] \\
\frac{\partial [D_\mathbf{x}\mathbf{h}]_{21}}{\partial x} = & \sum_{j=1}^M \left[2(y-y^*_j)\frac{\partial \sigma_j}{\partial x} \frac{\partial \nu_j}{\partial x} + (y-y^*_j)\sigma_j\frac{\partial^2 \nu_j}{\partial x^2} \right. \nonumber \\
& \left. + (y-y^*_j)(\nu_j-1)\frac{\partial^2 \sigma_j}{\partial x^2}  \right] \\
\frac{\partial [D_\mathbf{x}\mathbf{h}]_{21}}{\partial y} = & \sum_{j=1}^M \left[\sigma_j \frac{\partial \nu_j}{\partial x} + (y-y^*_j)\frac{\partial \sigma_j}{\partial y}\frac{\partial \nu_j}{\partial x} + (y-y^*_j)\sigma_j\frac{\partial^2 \nu_j}{\partial x \partial y} \right. \nonumber \\
& \left. + (\nu_j-1)\frac{\partial \sigma_j}{\partial x} + (y-y^*_j) \frac{\partial \sigma_j}{\partial x} \frac{\partial \nu_j}{\partial y} + (y-y^*_j)(\nu_j-1)\frac{\partial^2 \sigma_j}{\partial x \partial y}  \right] \\
\frac{\partial [D_\mathbf{x}\mathbf{h}]_{22}}{\partial x} = & \sum_{j=1}^M \left[ \sigma_j\frac{\partial \nu_j}{\partial x} + (\nu_j-1)\frac{\partial \sigma_j}{\partial x} + (y-y^*_j)\frac{\partial \sigma_j}{\partial x}\frac{\partial \nu_j}{\partial y} + (y-y^*_j)\sigma_j\frac{\partial^2 \nu_j}{\partial x \partial y} \right. \nonumber \\
& \left. + (y-y^*_j)\frac{\partial \sigma_j}{\partial y}\frac{\partial \nu_j}{\partial x} + (y-y^*_j)(\nu_j-1)\frac{\partial^2 \sigma_j}{\partial x \partial y} \right] \\
\frac{\partial [D_\mathbf{x}\mathbf{h}]_{22}}{\partial y} = & \sum_{j=1}^M \left[ 2\sigma_j \frac{\partial \nu_j}{\partial y} + 2(\nu_j-1)\frac{\partial \sigma_j}{\partial y} + 2(y-y^*_j)\frac{\partial \sigma_j}{\partial y}\frac{\partial \nu_j}{\partial y} \right. \nonumber \\
& \left. + \sigma_j(y-y^*_j)\frac{\partial^2 \nu_j}{\partial y^2} + (y-y^*_j)(\nu_j-1)\frac{\partial^2 \sigma_j}{\partial y^2} \right]
\end{align}
In the expressions above, we use elements of the Hessians
\begin{align}
\nabla^2 \sigma_j(\mathbf{x}) = & \eta''(\beta_j(\mathbf{x}))(\nabla \beta_j(\mathbf{x}))(\nabla \beta_j(\mathbf{x}))^\top + \eta'(\beta_j(\mathbf{x}))\nabla^2\beta_j(\mathbf{x}) \\
\nabla^2 \nu_j(\mathbf{x}) = & \frac{3 \rho_j}{||\mathbf{x}-\mathbf{x}^*_j||^5}(\mathbf{x}-\mathbf{x}^*_j)(\mathbf{x}-\mathbf{x}^*_j)^\top - \frac{\rho_j}{||\mathbf{x}-\mathbf{x}^*_j||^3}\mathbf{I}
\end{align}
Eventually we can calculate $v D_\mathbf{x}\xi \begin{bmatrix} \cos\psi & \sin\psi \end{bmatrix}^\top$, used in \eqref{eq:reference_input_2}, as follows 
\begin{equation}
v D_\mathbf{x}\xi \begin{bmatrix} \cos\psi \\ \sin\psi \end{bmatrix} = \frac{(\alpha_1\beta_1+\alpha_2\beta_2)v}{||\mathbf{e}(\mathbf{x},\psi)||^2} 
\end{equation}
with
\begin{align}
\alpha_1 = & -([D_\mathbf{x}\mathbf{h}]_{21}\cos\psi + [D_\mathbf{x}\mathbf{h}]_{22}\sin\psi) \\
\alpha_2 = & [D_\mathbf{x}\mathbf{h}]_{11}\cos\psi + [D_\mathbf{x}\mathbf{h}]_{12}\sin\psi \\
\beta_1 = & \frac{\partial [D_\mathbf{x}\mathbf{h}]_{11}}{\partial x} \cos^2\psi + \left( \frac{[D_\mathbf{x}\mathbf{h}]_{11}}{\partial y} + \frac{[D_\mathbf{x}\mathbf{h}]_{12}}{\partial x}\right) \sin\psi \cos\psi + \frac{\partial [D_\mathbf{x}\mathbf{h}]_{12}}{\partial y} \sin^2\psi \\
\beta_2 = & \frac{\partial [D_\mathbf{x}\mathbf{h}]_{21}}{\partial x} \cos^2\psi + \left( \frac{[D_\mathbf{x}\mathbf{h}]_{21}}{\partial y} + \frac{[D_\mathbf{x}\mathbf{h}]_{22}}{\partial x}\right) \sin\psi \cos\psi + \frac{\partial [D_\mathbf{x}\mathbf{h}]_{22}}{\partial y} \sin^2\psi
\end{align}

\end{document}